\documentclass[nohyperref]{article}

\usepackage{graphicx}	% Including figure files
\usepackage[utf8]{inputenc}
\usepackage[english]{babel}
\usepackage{amsthm}
\usepackage{amsfonts}
\usepackage{dsfont}
\usepackage{algorithmic}
\usepackage[dvipsnames,table,xcdraw]{xcolor}
\usepackage[colorlinks=true, linkcolor=black, citecolor=blue]{hyperref}
\usepackage{microtype}
\usepackage{subfigure}
\usepackage{booktabs} 
\usepackage{amsmath,amssymb}
\usepackage{physics}
\usepackage[skip=0pt]{caption}

\usepackage{amsmath,amsfonts,bm}

\newcommand{\BlackBox}{\rule{1.5ex}{1.5ex}}  % end of proof
\ifdefined\proof
    \renewenvironment{proof}{\par\noindent{\bf Proof\ }}{\hfill\BlackBox\\[2mm]}
\else
    \newenvironment{proof}{\par\noindent{\bf Proof\ }}{\hfill\BlackBox\\[2mm]}
\fi
 
\newtheorem{theorem}{Theorem}
\newtheorem{lemma}{Lemma} 
 
\newtheorem{remark}{Remark}

\newtheorem{definition}{Definition}

\usepackage[accepted]{style/icml2023}

\usepackage[capitalize,noabbrev]{cleveref}  % need to be after hyperref

\crefname{proposition}{Prop.}{Props.}
\crefname{definition}{Def.}{Defs.}
\crefname{lemma}{Lemma}{Lemmas}
\crefname{example}{Ex.}{Exs.}
\crefname{equation}{}{}
\crefname{section}{\S\hspace{-0.2em}}{\S\S}
\crefname{appendix}{\S\hspace{-0.2em}}{\S\S}
\crefname{subsection}{\S\hspace{-0.2em}}{\S\S}
\crefname{subsubsection}{\S\hspace{-0.2em}}{\S\S}
\crefname{figure}{Fig.}{Figs.}
\crefname{wrapfigure}{Fig.}{Figs.}
\crefname{corollary}{Cor.}{Cors.}
\crefname{table}{Table}{Tables}

\DeclareMathOperator{\vol}{vol}
\DeclareMathOperator{\ecp}{ECP}
\DeclareMathOperator{\covprob}{CP}
\newcommand{\cravg}[1]{\overline{#1}}

\newcommand{\revised}[1]{{{#1}}}

\icmltitlerunning{Sampling-Based Accuracy Testing of Posterior Estimators for General Inference}

\begin{document}

\twocolumn[
\icmltitle{Sampling-Based Accuracy Testing of Posterior Estimators for General Inference}

\icmlsetsymbol{equal}{*}

\icmlcorrespondingauthor{Pablo Lemos}{pablo.lemos@umontreal.ca}

\begin{icmlauthorlist}
\icmlauthor{Pablo Lemos}{mila,udem,ciela,cca,equal}
\icmlauthor{Adam Coogan}{mila,udem,ciela,equal}
\icmlauthor{Yashar Hezaveh}{mila,udem,ciela}
\icmlauthor{Laurence Perreault-Levasseur}{mila,udem,ciela}
\end{icmlauthorlist}

\icmlaffiliation{mila}{Mila -- Quebec AI Institute, Montreal, Quebec, Canada}
\icmlaffiliation{udem}{Universit\'e de Montr\'eal, Montreal, Quebec, Canada}
\icmlaffiliation{ciela}{CIELA Institute, Montreal, Quebec, Canada}
\icmlaffiliation{cca}{Flatiron Institute Center for Computational Mathematics, 162 5th Ave, 3rd floor, New York, NY 10010, USA}

\icmlkeywords{}

\vskip 0.3in
]

\printAffiliationsAndNotice{}  % leave blank if no need to mention equal contribution

\begin{abstract}
    Parameter inference, i.e. inferring the posterior distribution of the parameters of a statistical model given some data, is a central problem to many scientific disciplines. %\revised{[@Pablo to rewrite] Posterior inference with generative models is an alternative to methods such as Markov Chain Monte Carlo, both for likelihood-based and simulation-based inference.} 
    \revised{Generative models can be used as an alternative to Markov Chain Monte Carlo methods for conducting posterior inference, both in  likelihood-based and simulation-based problems.}
    However, assessing the accuracy of posteriors encoded in generative models is not straightforward. In this paper, we introduce `Tests of Accuracy with Random Points' (TARP) coverage testing as a method to estimate coverage probabilities of generative posterior estimators.
    Our method differs from previously-existing coverage-based methods, which require posterior evaluations. We prove that our approach is necessary and sufficient to show that a posterior estimator is accurate. We demonstrate the method on a variety of synthetic examples, and show that TARP can be used to test the results of posterior inference analyses in high-dimensional spaces. We also show that our method can detect inaccurate inferences in cases where existing methods fail.
\end{abstract}

\section{Introduction}

\revised{The task of parameter inference,} i.e.\ determining the values of unknown parameters $\theta$ in a statistical model \revised{consistent with} observed data $x$, is a ubiquitous \revised{task in scientific analyses.}
\revised{While multiple well-established approaches such as Markov-chain Monte Carlo (MCMC), variational inference (VI) and nested sampling~\citep{10.1214/06-BA127} already exist, there has been a recent shift towards applying machine learning for posterior inference amortized over different observations~\citep[e.g. ][]{2018JCoPh.366..415Z,2018arXiv180304765P, 2020arXiv200601490C, NEURIPS2020_322f6246, doi:10.1021/acs.jpca.9b08723}. This approach involves training a model (typically a neural network) to approximate the true posterior distribution as a function of the observation.
% by minimizing the difference between the estimated and desired posterior distribution.
The goal is to efficiently infer the posterior for new data, eliminating the need for costly MCMC runs for each new observation.} 

Simulation-based inference~\citep[SBI, e.g. ][]{cranmer2020frontier}, also known as likelihood-free inference (LFI) or implicit likelihood inference (ILI), has gained significant popularity in recent years \citep[e.g. ][]{ong2018likelihood, PerreaultLevasseur:2017ltk, 10.7554/eLife.56261, PhysRevLett.127.241103,alsing2019fast,Wagner-Carena:2020yun,Legin:2021zup,2021NatRP...3..305B, Coogan:2020yux, Montel:2022fhv, Coogan:2022cky, Brehmer:2019jyt, chen2020neural, Mishra-Sharma:2021oxe, Karchev:2022xyn, Hermans:2020skz, AnauMontel:2022ppb, https://doi.org/10.48550/arxiv.2005.07062, https://doi.org/10.48550/arxiv.2109.14275, Karchev:2022ycy, 2022arXiv220306481R}. SBI does not require an explicit expression for the likelihood, and instead merely relies on having a simulator to generate training data. The SBI framework allows handling complex, high-dimensional data and models that are difficult or intractable to analyze using traditional likelihood-based methods.

Early developments of SBI include the introduction of Rejection Approximate Bayesian Computation (ABC)~\citep{10.1214/aos/1176346785, pritchard1999population, beaumont2002approximate, marjoram2003markov, fearnhead2012constructing}, but today SBI has evolved to encompass more powerful, neural network-powered, amortized methods, such as Neural Ratio Estimation (NRE)~\cite{cranmer2015approximating, thomas2022likelihood, hermans2020likelihood, durkan2020contrastive, Miller:2022haf}; Neural Posterior Estimation (NPE)~\citep{rezende2015variational, papamakarios2016fast, 2018arXiv180509294L, lueckmann2017flexible, greenberg2019automatic} and Neural Likelihood Estimation (NLE)~\citep{price2018bayesian, papamakarios2019sequential, frazier2022bayesian}. Recently there has been substantial interest in applying SBI in high-dimensional parameter spaces. Generative models, such as Generative Adversarial Networks \revised{GANs}~\citep{goodfellow2014generative}, Normalizing Flows~\citep{dinh2014nice, rezende2015variational, papamakarios2021normalizing},  Variational Autoencoders~\citep{kingma2013auto} and Score-Based/Diffusion Models~\citep{song2020score, DBLP:journals/corr/abs-2006-11239, DBLP:journals/corr/Sohl-DicksteinW15}, are powerful ways to encode approximate posteriors in such settings.

\revised{Convergence} tests for MCMC methods, such as the Gelman-Rubin statistic~\citep{gelman1992inference}, the effective sample size and the integrated autocorrelation time, are well-established. \revised{However, these assess the diversity of samples rather than directly guaranteeing that the posterior is being sampled correctly.} For SBI, testing for accuracy of the estimated posterior is often performed using coverage probabilities (but see also~\citet{2017arXiv170604599G}), relying on the evaluation of the density of the posteriors.~\citep{schall2012empirical, prangle2013diagnostic, cranmer2020frontier, hermans2021averting}. Coverage probabilities measure the proportion of the time that a certain interval contains the true parameter value. 
\revised{However, coverage probability calculations based on evaluations of the learned posterior distributions are not applicable to samples obtained from those generative models where such evaluations are not available, such as GANs and diffusion models.}
%However, coverage probability calculations based on evaluations of the learned posterior distributions are not applicable to samples obtained from a generative model, where such evaluations are not available. \revised{Weaken the previous sentence based on ref. 1?} 
Furthermore, and more importantly, these coverage probability tests are a necessary but not sufficient diagnostic to assess the accuracy of the estimated posterior. 

\revised{Other methods have been suggested as alternative validations for SBI~\citep{lueckmann2021benchmarking, dalmasso2020conditional, deistler2022truncated}. For example, Simulation-Based Callibration (SBC)~\citep{talts2018validating}, proposes an interesting technique that uses only samples, but can only be used for one-dimensional posteriors and not the full-dimensional space. Another method, proposed by \citet{linhart2022validation}, is an efficient way to assess posterior accuracy but is designed specifically for normalizing flows and cannot be applied in other inference settings. None of these methods can be applied to assess the accuracy of inference for high-dimensional variables. 
%Other methods such as the Conditional Density Estimation (CDE) loss~\citep{} and the Probability Integral Transform (PIT)~\citep{} rely on evaluations of the posterior estimator, and therefore cannot be applied for \revised{those} generative models on arbitrary dimensions which provide only samples. 
}

% Although other works have suggested alternative validation methods for SBI~\citep{lueckmann2021benchmarking, dalmasso2020conditional, deistler2022truncated}, none of these can be applied for \revised{those} generative models on arbitrary dimensions 
% %when 
% \revised{for which}
% we do not have access to posterior evaluations.
% \revised{
% Exceptions to that are Simulation-Based Callibration (SBC)~\citep{talts2018validating}, which uses only samples, but can only be used for one-dimensional posteriors, and not the full-dimensional space as our work does, and~\citep{linhart2022validation}, which does not require posterior evaluations, but is designed specifically for normalizing flows, and therefore cannot be applied to other generative models.
% }

The goal is this paper is to introduce a framework for testing the accuracy of parameter inference using only samples from the true joint distribution of data $x$ and the parameters of interest $\theta$, $p(x, \theta)$, and samples from the estimated posterior distribution $\hat{p}(\theta | x)$.

\revised{Our novel contributions are a proof of the necessary and sufficient conditions to verify the accuracy of posterior estimators through coverage checks (\cref{theorem:main}), along with a methods that practically implements this theorem (\cref{sec:drp_method}). We begin by introducing all necessary notation in~\cref{sec:formalism}. We then introduce our method in~\cref{sec:method}, present our experiments in~\cref{sec:experiments}, and summarize our findings in~\cref{sec:conc}. Our code is available at \href{https://github.com/Ciela-Institute/tarp}{\texttt{https://github.com/Ciela-Institute/tarp}}.}

\section{Formalism}\label{sec:formalism}

In this section, we introduce some basic concepts and build up to our key theoretical result (\cref{theorem:main}). The coverage testing procedure introduced in the following section is essentially a practical implementation of this theorem.

\subsection{Notation}

As stated in the introduction, we are interested continuous-valued parameters $\theta \in U \subset \mathbb{R}^n$ and observations $x \in V \subset \mathbb{R}^m$ taken from (subsets of) Euclidean space, with joint density $p(\theta, x)$. We denote our posterior estimator by $\hat{p}(\theta | x)$ (which could be a neural network or MCMC sampler, for example) and assume we can also use it to generate samples of $\theta$.

With these preliminaries, we make two basic definitions:
% \amc{can we do without `calibration'?}
% \TODO{Perhaps `accurate' is a better word?}
\begin{definition}\label{def:calibrated}
    A posterior estimator $\hat{p}(\theta| x)$ is {\bf accurate} if 
    \begin{equation}
        \hat{p}(\theta | x) = p(\theta | x ) \quad \forall (x, \theta) \sim p(x, \theta) \, .
    \end{equation}
\end{definition}

\begin{definition}\label{def:gen}
    A \textbf{credible region generator} $\mathcal{G}: \hat{p}, \alpha, x \mapsto W \subset U$ for a given credibility level \revised{$1 - \alpha$} and observation $x$ is a function satisfying
    \begin{equation}\label{eq:ECP_gen}
        \int_{\mathcal{G}(\hat{p}, \alpha, x)} \dd{\theta} \hat{p}(\theta | x) = 1 - \alpha \, .
    \end{equation}
\end{definition}
Note that there are an infinite number of such generators. A commonly-used one is the highest-posterior density region generator, defined in~\cref{sec:prev_work}.

Next, we introduce two central definitions for this work, adapted from \citet{hermans2021averting} (henceforth H21).
\begin{definition}\label{def:CP}
    The \textbf{coverage probability} for a generator $\mathcal{G}$, credibility level \revised{$1 - \alpha$} and datum $x$ is
    \begin{equation}\label{eq:CP}
        \covprob(\hat{p}, \alpha, x, \mathcal{G}) = \mathbb{E}_{p(\theta | x)} \left[ \mathds{1}\left(\theta \in  \mathcal{G}(\hat{p}, \alpha, x) \right) \right] \, .
    \end{equation}
\end{definition}

\begin{definition}\label{def:ecp}
    The \textbf{expected coverage probability} for a generator $\mathcal{G}$ and credibility level $\alpha$ is the coverage probability averaged over the data distribution:
    \begin{equation}\label{eq:ecp}
        \ecp(\hat{p}, \alpha, \mathcal{G}) = \mathbb{E}_{p(x)} \left[ \covprob(\hat{p}, \alpha, x, \mathcal{G}) \right] \, .
    \end{equation}
\end{definition}

\subsection{Coverage probability}

We now demonstrate some basic facts about estimators with correct coverage probabilities. We begin with a straightforward result:
\begin{theorem}
    The posterior has coverage probability $\operatorname{CP}(p, \alpha, x, \mathcal{G}) = 1 - \alpha$ for all values of $x$ and any credible region generator \revised{$\mathcal{G}(p, \alpha, x)$}.
\end{theorem}

\begin{proof}
    Substituting $\hat{p}(\theta | x) = p(\theta | x)$, the definition of coverage probability becomes:
    \begin{equation}
    \begin{split}
        \covprob(p, \alpha, x, \mathcal{G}) &= \mathbb{E}_{p(\theta | x)}\left[ \mathds{1}(\theta \in \mathcal{G}(p, \alpha, x) \right] \\
        &= \int_{\mathcal{G}(p, \alpha, x)} \dd{\theta} p(\theta | x) \\
        &= 1 - \alpha \, ,
    \end{split}
    \end{equation}
    where the last line follows from the definition of a credible region.
\end{proof}
It follows trivially from this that the posterior has $\ecp(p, \alpha, x, \mathcal{G}) = 1 - \alpha$ as well.

Next, we prove the more interesting reverse direction of this theorem, \revised{which requires introducing another type of credible region generator.
\begin{definition}
    A \textbf{positionable credible region generator} $\mathcal{P}_{\theta_r}(\hat{p}, \alpha, x)$ generates credible regions positioned at $\theta_r$, a freely-chosen point in parameter space, in the sense that
    \begin{equation}
        \lim_{\alpha \to 1} \mathcal{P}_{\theta_r}(\hat{p}, \alpha, x) = \{ \theta_r \}
    \end{equation}
    for all $x$ and $\theta_r$. The regions' shapes are not important: they could be, for example, balls or hypercubes. % [@Laurence: maybe these should be simply-connected?]
\end{definition}
}

% I think this was wrong before: didn't include \theta_r as the argument to \bar{f}.
\revised{
Lastly, we denote the average of a function $f(\cdot)$ over a credible region $\Theta$ positioned at $\theta_r$ as
\begin{equation}
    \cravg{f(\theta_r)}(\Theta) := \frac{1}{\vol[ \Theta ]} \int_{\Theta} \dd{\theta} f(\theta) \, .
\end{equation}
When $f(\cdot)$ is a probability density function, $\cravg{f(\cdot)}(\Theta)$ is as well, since it is the convolution of $f(\cdot)$ with the density $\mathds{1}(\theta \in \Theta)$.  % thanks for the help, ChatGPT!
}

\begin{theorem}
    Suppose the coverage probability of a posterior estimator is equal to $1 - \alpha$ for a positionable credible region generator $\mathcal{P}_{\theta_r}$ for all $\theta_r$, $x$ and $\alpha$. Further, suppose that \revised{$\hat{p}(\cdot | x)$ and $p(\cdot | x)$ are both continuous on their domains}. Then $\hat{p}(\cdot | x) = p(\cdot | x)$.
\end{theorem}

\begin{proof}
    Define $\Theta := \mathcal{P}_{\theta_r}(\hat{p}, \alpha, x)$ for clarity.
    
    The integral in the definition of the coverage probability can be written as
    \begin{equation}
    \begin{split}
        \covprob(\hat{p}, \alpha, x, \mathcal{P}_{\theta_r}) &= 1 - \alpha \\
        &= \int_{\Theta} \dd{\theta} p(\theta | x) \\
        &= \vol[ \Theta ] \, \cravg{p(\cdot | x)}(\Theta) \, ,
    \end{split}
    \end{equation}
    where first equality follows by assumption. Since we've assumed $\hat{p}(\cdot | x)$ has support everywhere \revised{that $p(\cdot | x)$ has support}, the volume of the credible region is positive. By the definition of a credible region, we also have
    \begin{equation}
        1 - \alpha = \int_{\Theta} \dd{\theta} \hat{p}(\theta | x) = \vol[ \Theta ] \, \cravg{\hat{p}(\cdot | x)}(\Theta) \, .
    \end{equation}
    Setting this equal to the previous expression yields $\cravg{\hat{p}(\cdot | x)}(\Theta) = \cravg{p(\cdot | x)}(\Theta)$, which holds for all $\theta_r$ and $x$ by assumption. Taking $\alpha \to 1$ (i.e., making $\Theta$ small) gives the desired result.
\end{proof}

\subsection{Expected coverage probability}

\begin{figure*}[th]
    \centering
    \includegraphics[width=\linewidth]{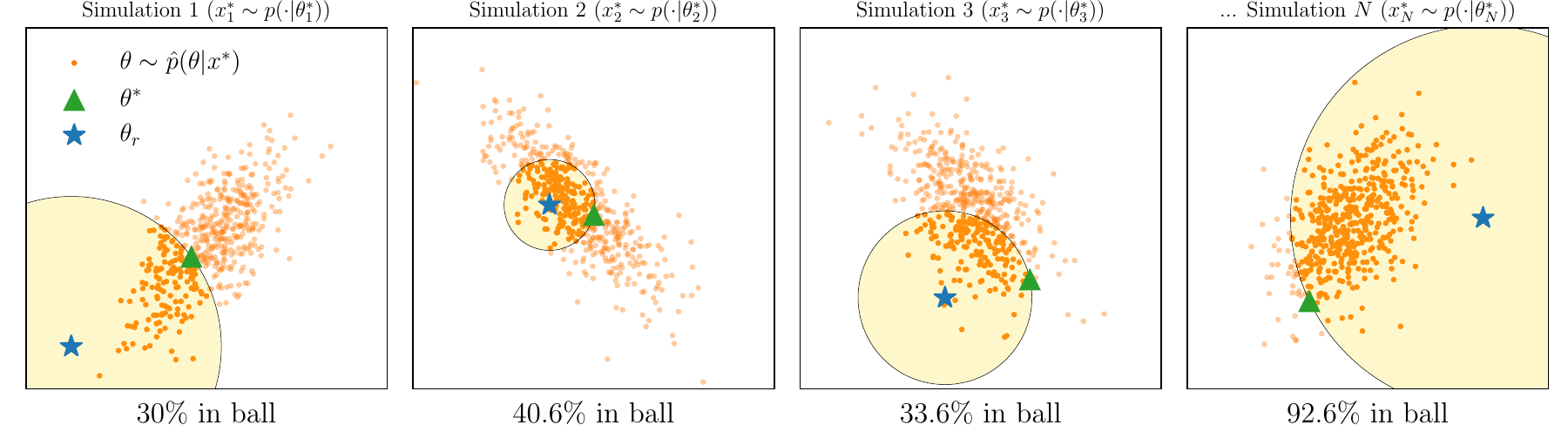}
    \caption{
        A graphical illustration of the proposed \revised{coverage test for assessing the quality of a posterior estimator $\hat{p}$. Given a set of simulations (panels), we draw samples from the posterior estimator (orange points). We sample a reference parameter point $\theta_r$, and determine the fraction of points $f$ falling within a ball centered on $\theta_r$ extending to the true parameter point $\theta^*$ used to generate the simulation (ball indicated in yellow, $f$ indicated below each panel). Our coverage test aggregates the statistics of $f$, providing a necessary and sufficient way to guarantee the accuracy of $\hat{p}$.}
        % for coverage probabilities.
        % {\it Step 1}: We use each simulation from the validation set to generate a number of samples $n$ from the posterior estimator.
        % {\it Step 2} We pick a random location in parameter space as our reference $\theta_r$.
        % {\it Step 3}: We calculate the distance $d_i$ between each sample and $\theta_r$
        % {\it Step 4}: We calculate the distance $d^*$ between the true value of the parameters $\theta^*$, and $\theta_r$ (bottom right panel). 
        % \revised{Explain procedure in more detail.}
    }
    \label{fig:method}
\end{figure*}

The previous result is still not very useful, since it is computationally very expensive to calculate the coverage probability of a posterior estimator. Practically, doing so requires producing histograms of the samples from $p(\theta, x)$ in $x$, which may be high-dimensional. However, as pointed out in H21, it's much simpler to compute the \emph{expected} coverage probability.

The next theorem is our main theoretical result: correct expected coverage is enough to verify the posterior estimator is accurate, as long as it is correct for any function $\theta_r(x)$ defining the positions of the credible regions.

\begin{theorem}\label{theorem:main}
    Suppose the expected coverage probability of $\hat{p}$ is equal to $1 - \alpha$ for a positionable credible region generator $\mathcal{P}_{\theta_r}$ for all $\alpha$, $x$, and $\theta_r(\cdot)$ assigning a position to the credible regions as a function of $x$. Further suppose that $\hat{p}(\cdot | x)$ \revised{has support everywhere that $p(\cdot | x)$ has support, and that both functions are continuous on their domains.} Then $\hat{p}(\cdot | x) = p(\cdot | x)$.
\end{theorem}

\begin{proof}
    Again, let $\Theta := \mathcal{P}_{\theta_r}(\hat{p}, \alpha, x)$ for clarity.
    
    First, we leverage the definition of credible regions to find an expression for the volume of $\Theta$:
    \begin{equation}
        1 - \alpha = \int_{\Theta} \dd{\theta} \hat{p}(\theta | x) = \vol[ \Theta ] \, \cravg{p(\cdot | x)}(\Theta) \, ,
    \end{equation}
    which implies
    \begin{equation}
        \vol[ \Theta ] = \frac{1 - \alpha}{\cravg{p(\cdot | x)}(\Theta)} \, .
    \end{equation}
    This allows us to expand and simplify the expression for the expected coverage:
    \begin{align}
    \begin{split}
        \ecp(\hat{p}, \alpha, \mathcal{P}_{\theta_r}) &= 1 - \alpha \\
        &= \int \dd{x} p(x) \int_{\Theta} \dd{\theta} p(\theta | x) \\
        &= \int \dd{x} p(x) \, \vol[ \Theta ] \, \cravg{p(\cdot | x)}(\Theta) \\
        &= (1 - \alpha) \int \dd{x} p(x) \, \frac{\cravg{p(\cdot | x)}(\Theta)}{\cravg{\hat{p}(\cdot | x)}(\Theta)} \, .
    \end{split}
    \end{align}
    Canceling the factors of $1 - \alpha$ gives that the integral in the last line is equal to $1$.

    By assumption, this holds for \emph{any} choice of position function $\theta_r(x)$. We can therefore take the functional derivative of the integral with respect to $\theta_r(x)$. Recalling that the averages in the integrand depend on $\theta_r$, we obtain
    \begin{align}
        0 &= \frac{\delta}{\delta \theta_r(x)} \int \dd{x} p(x) \, \frac{\cravg{p(\cdot | x)}(\Theta)}{\cravg{\hat{p}(\cdot | x)}(\Theta)} \\
        &= \int \dd{x} \delta\theta_{r,i}(x) \, p(x) \pdv{}{\theta_{r,i}} \left( \frac{\cravg{p(\cdot | x)}(\Theta)}{\cravg{\hat{p}(\cdot | x)}(\Theta)} \right) \\
        &= \int \dd{x} \delta\theta_{r,i}(x) \frac{\cravg{p(\cdot | x)}(\Theta) \, p(x)}{\cravg{\hat{p}(\cdot | x)}(\Theta)} \nonumber\\
        & \qquad\qquad \times\left[ \pdv{\log \cravg{p(\cdot | x)}(\Theta)}{\theta_{r,i}} - \pdv{\log \cravg{\hat{p}(\cdot | x)}(\Theta)}{\theta_{r,i}} \right] \, ,
    \end{align}
    where the $i$ subscript indexes the components of $\theta_r$. Since this expression must hold for all variations $\delta\theta_{r,i}$, the integrand must evaluate to zero (i.e., the Euler-Lagrange equation must be satisfied). By assumption, the factor outside the braces in the integrand is nonzero, implying
    \begin{equation}
        \pdv{\log \cravg{p(\cdot | x)}(\Theta)}{\theta_{r,i}} = \pdv{\log \cravg{\hat{p}(\cdot | x)}(\Theta)}{\theta_{r,i}} \, .
    \end{equation}
    This implies $\log \cravg{p(\cdot | x)}(\Theta) = \log \cravg{\hat{p}(\cdot | x)}(\Theta) + c(x)$, for some $x$-dependent integration constant $c$. But since the functions inside the logarithms themselves densities, we have $c(\cdot) = 0$. Taking the limit $\alpha \to 1$ gives $\hat{p}(\theta | x) = p(\theta | x)$.
\end{proof}

The coverage testing method we will introduce in the next section is effectively a practical implementation of this theorem.

\section{Our method}\label{sec:method}

With our main theoretical result proven (c.f.\ \cref{theorem:main}), in this section we use it to first explain the blind spots of typical coverage probability calculations and then introduce our new coverage checking procedure.

\subsection{High posterior density coverage testing}\label{sec:prev_work}

\begin{figure*}[t]%[tbhp]
    \centering
    \includegraphics[width=.99\linewidth]{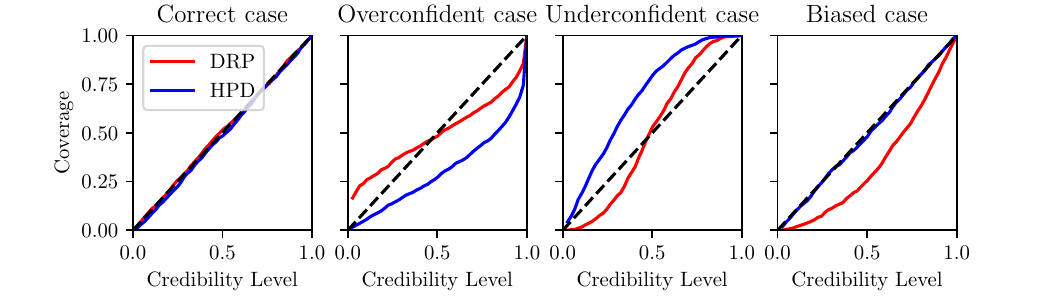}
    \caption{
    Results on the Gaussian toy model for all four cases described in~\cref{sec:toy}. The red line shows the method presented in this paper, while the blue shows the HPD region.
    }
    \label{fig:toy_gaussian}
\end{figure*}

\begin{algorithm*}[tb]
    \caption{Calculation of ${\rm ECP}(\hat{p}, \alpha, \mathcal{H})$ using highest posterior density regions, from a set of simulations $\left\{\theta_i, x_i\right\}$, $i \in [1, N_{\rm sims}$]}\label{alg:ecp_hpd}
    \begin{algorithmic}
        \STATE {Generate $n$ samples $\left\{ \theta_{i j} \right\} \sim \hat{p}(\theta | x_i)$ for each simulation $x_i$.} 
        \FOR {$i \gets 1$ to $N_{\rm sims}$}
        \STATE {$f_i = (1 / $n$) \cdot \sum_{j=1}^n  \mathds{1} \left[ \hat{p} (\theta_{ij} | x_i) < \hat{p} (\theta^*_i | x_i) \right]$} %\Comment{Excess probability}
        \ENDFOR
        \STATE {${\rm ECP}(\hat{p}, \alpha, \mathcal{H}) = (1 / N_{\rm sims}) \sum_{i=1}^{N_{\rm sims}} \mathds{1} \left( f_i < 1 - \alpha\right)$}
    \end{algorithmic}
\end{algorithm*}

Before introducing the proposed method, we first discuss HPD coverage.

\revised{
\begin{definition}
    We define the HPD credible region generator $\mathcal{H} (\hat{p}, \alpha, x)$ as the generator that produces the region with mass $1 - \alpha$ occupying the smallest possible volume in $U$\footnote{Note this is ill-defined for the uniform density function.}.
\end{definition}
Note that this is not a positionable credible region generator.
This can be used combined with~\cref{def:ecp} to calculate High-Posterior Density Expected Coverage Probabilities (HPD ECPs). HPD ECPs are often used to assess coverage \cite{hermans2021averting, rozet2021arbitrary, Miller2022, deistler2022truncated, 2020JOSS....5.2505T}. 
}

\revised{
Perhaps the most intuitive way of calculating expected coverage probability using HPD regions is to compute such a region for all possible values of $\alpha$,\footnote{
    \revised{Note that previous works such as \citet{PerreaultLevasseur:2017ltk} have attempted to perform accuracy testing from a handful of values of $\alpha$. This test is not nearly as restrictive as scanning over all possible values of $\alpha$ as is typically used for coverage testing.}
} then calculate the expected coverage using~\cref{eq:CP}. In practice, however, there is a more efficient calculation of expected coverage probabilities, which is derived from the following result:
\begin{remark}
    A pair ($\theta^*, x^*$), and a posterior estimator $\hat{p}(\theta | x)$ uniquely define a HPD confidence region as: 
    \begin{equation}\label{eq:theta_hpd}
        \left\{ \theta \in U \ | \ \hat{p}(\theta | x^*) \geq \hat{p}(\theta^* | x^*) \right\}.
    \end{equation}
    This, in turn, defines a corresponding {\bf HPD confidence level} $1 - \tilde{\alpha}_{\rm HPD}(\hat{p}, \theta^*, x^*)$, as the integral of $\hat{p}(\theta | x)$ over that region.
    %We denote the generator of HPD credible regions as $\mathcal{H}$.
\end{remark}
}

% HPD credible regions are often used to assess coverage \cite{hermans2021averting, rozet2021arbitrary, Miller2022, deistler2022truncated, 2020JOSS....5.2505T}. Perhaps the most intuitive way of calculating expected coverage probability using HPD regions is to compute such a region for every value of $\alpha$ \footnote{\revised{Note that previous works, such as~\cite{PerreaultLevasseur:2017ltk} have attempted to perform accuracy testing from a handful of values of $\alpha$, but this test is not nearly as restrictive as the "diagonal plots" normally used.}} 
% , then calculate the expected coverage using~\cref{eq:CP}. In practice, however, there is a more efficient calculation of expected coverage probabilities, which is derived from the following:
% \begin{definition}
%     \revised{Make this into a generator?}
%     A pair ($\theta^*, x^*$), and a posterior estimator $\hat{p}(\theta | x)$ uniquely define a HPD confidence region as: 
%     \begin{equation}
%         \Theta_{\rm HPD} \left( x^*, \theta^*, \hat{p} \right) := \left\{ \theta \in U \ | \ \hat{p}(\theta | x^*) \geq \hat{p}(\theta^* | x^*) \right\}.
%     \end{equation}
%     This, in turn, defines a corresponding {\bf HPD confidence level} $1 - \tilde{\alpha}_{\rm HPD}(\hat{p}, \theta^*, x^*)$ \revised{[Clarify for ref. 2]}. We denote the generator of HPD credible regions as $\mathcal{H}$.
% \end{definition}
% \revised{Recheck this.} 
We can then rederive an important result for this HPD confidence level:
\begin{lemma}\label{lemma:hpd}
    We can calculate the ECP of the $1 - \alpha$ highest posterior density regions as:
    \begin{equation}\label{eq:ECP_alt}
        \ecp(\hat{p}, \alpha, \mathcal{H}) = \mathbb{E}_{p(\theta, x)} \left[ \mathds{1}\left( \tilde{\alpha}_{\rm HPD}(\hat{p}, \theta, x) \geq \alpha \right) \right].
    \end{equation}
\end{lemma}

\begin{proof}
    Firstly, we notice that:
    \revised{
    \begin{equation}
        \theta^* \in \mathcal{H} \left( \hat{p}, \alpha, x^* \right) \Leftrightarrow \tilde{\alpha}_{\rm HPD}(\hat{p}, \theta^*, x^*) \geq \alpha.
    \end{equation}
    This follows from the fact that, if $\theta^* \in \mathcal{H} \left( \hat{p}, \alpha, x^* \right)$, then the HPD confidence region defined by $( \theta^*, x^* )$ is contained in $\mathcal{H} \left( \hat{p}, \alpha, x^* \right)$.
    }
    % \begin{equation}
    %     \theta^* \in \Theta_{\rm HPD} \left( x^*, \alpha, \hat{p} \right) \Leftrightarrow \tilde{\alpha}_{\rm HPD}(\hat{p}, \theta^*, x^*) \geq \alpha.
    % \end{equation}
    % This follows from the fact that, if $\theta^* \in \Theta_{\rm HPD} \left( x^*, \alpha, \hat{p} \right)$, then the HPD confidence region defined by $( \theta^*, x^* )$ is contained in $\Theta_{\rm HPD} \left( x^*, \alpha, \hat{p} \right)$.
    
    Then, from~\cref{eq:ecp}, it follows that~\cref{eq:ECP_alt} is true.
\end{proof}

This result can be used in practice to calculate the HPD ECP from samples of the true joint distribution $p(\theta, x)$, as shown in~\cref{alg:ecp_hpd}. As previously discussed, this algorithm requires explicit evaluations of the posterior estimator. We try to provide more intuitive connections between both definitions in~\cref{app:connection}.

As is well-known in the literature, estimating the ECP with HPD regions is not enough to demonstrate a posterior estimator is accurate. \Cref{theorem:main} reveals why: by definition, the HPD region generator is not positionable. Positionability is critical to the proof of the theorem, since it requires varying the position function $\theta_r(x)$.

To concretely demonstrate how considering only HPD coverage can fail, we consider the interesting case discussed in H21 of $\hat{p}(\theta | x) = p(\theta)$. From the definition of ECP, 
\begin{equation}
\begin{split}
    \ecp(\hat{p}, \alpha, \mathcal{H}) &= \mathbb{E}_{p(x, \theta)}[ \mathds{1}(\theta \in \mathcal{H}(\hat{p}, \alpha) ] \\
    &= \mathbb{E}_{p(\theta)}[ \mathds{1}(\theta \in \mathcal{H}(\hat{p}, \alpha) ] \\
    &= \int_{\mathcal{H}(\hat{p}, \alpha)} \dd{\theta} p(\theta) \\
    &= 1 - \alpha \, .
\end{split}
\end{equation}
In the second line, we used the fact that HPD generator is independent of $x$ in this case $\mathcal{H}(\hat{p}, \alpha, x) = \mathcal{H}(\hat{p}, \alpha)$. We recognize the third line as the definition of a credible region for the prior, yielding the fourth line. This means that $\hat{p}(\theta | x)$ has perfect HPD ECP in this case. 

We now introduce a coverage testing method that remedies such blind spots.

\subsection{Distance to random point coverage testing}\label{sec:drp_method}

\begin{figure*}[t]%[tbhp]
    \centering
    \includegraphics[width=.99\linewidth]{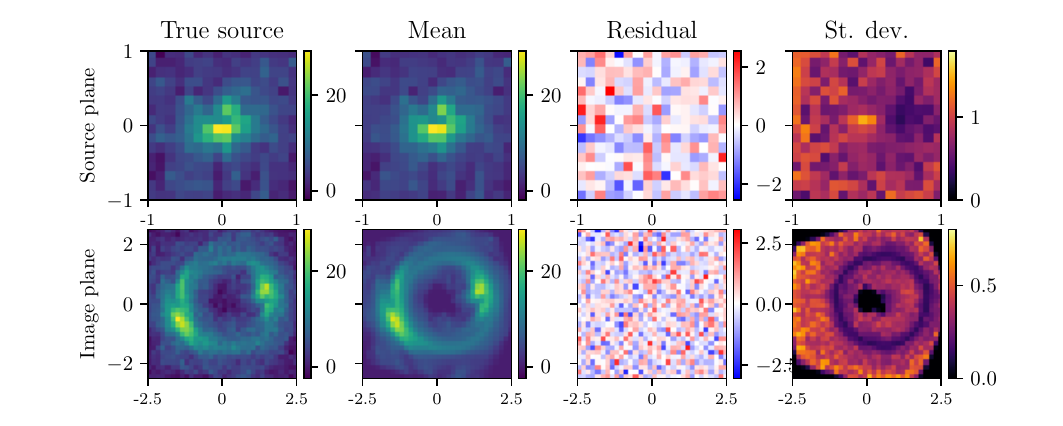}
    \caption{
    An example of one of the lensing simulations performed. The top panels show the (latent) source plane that we are trying to infer, while the bottom panels show the distorted images. From left to right, the plot shows the truth, mean, and standard deviation of the samples from the posterior estimator (in the case of this figure, the `exact' estimator), and the residuals. The noise in the observations is set to 1 on the color scales shown here.
    }
    \label{fig:lensing_example}
\end{figure*}

\begin{algorithm*}[tb]
    \caption{Calculation of ${\rm ECP}(\hat{p}, \alpha, \mathcal{D}_{\theta_r})$ using the TARP method, using a set of simulations $\left\{\theta_i, x_i\right\}$, $i \in \{1, \dots, N_{\rm sims} \}$, parameter distance metric $d: U \times U \rightarrow \mathbb{R}^{\geq 0}$ and reference point sampling distribution $\tilde{p}(\cdot | x)$.}\label{alg:ecp_name}
    \begin{algorithmic}
        \STATE {Generate $n$ samples $\left\{ \theta_{i j} \right\} \sim \hat{p}(\theta | x_i)$ for each simulation $x_i$, where $j = \{ 1, \dots, n \}$.} 
        \FOR {$i \gets 1$ to $N_{\rm sims}$}
        \STATE {$\theta_r \sim \tilde{p}(\theta_r | x)$} \COMMENT{Generate reference point}
        \STATE {$f_i = (1 / $n$) \cdot \sum_{j=1}^n  \mathds{1} \left[ d (\theta_{ij}, \theta_r) < d (\theta^*_i, \theta_r) \right]$} %\Comment{Excess probability}
        \ENDFOR
        \STATE {${\rm ECP}(\hat{p}, \alpha, \mathcal{D}_{\theta_r}) = (1 / N_{\rm sims}) \sum_{i=1}^{N_{\rm sims}} \mathds{1} \left( f_i < 1 - \alpha \right)$}
    \end{algorithmic}
\end{algorithm*}

The method proposed here generates spherical credible regions around position $\theta_r$:

\revised{
\begin{definition}
    Given a distance metric $ d : U \times U \rightarrow R$, We define the generator of TARP regions $\mathcal{D}_{\theta_r} (\hat{p}, \alpha, x, d)$ as the positionable generator that produces credible regions of credibility level $1 - \alpha$:
    \begin{equation}
        \mathcal{D}_{\theta_r} (\hat{p}, \alpha, x, d) := \left\{ \theta \in U \ | \ d(\theta, \theta_r) \leq R(\hat{p}, \alpha, x) \right\},
    \end{equation}
    where $R(\hat{p}, \alpha, x)$ is such that~\cref{eq:ECP_gen} is satisfied. 
\end{definition}
}

\revised{
From this result, and similarly to the previous section, a key result follows:
}

\begin{remark}
     A pair ($\theta^*, x^*$), and a posterior estimator $\hat{p}(\theta | x)$ uniquely define a \textbf{TARP}\footnote{TARP is short for "Test of Accuracy with Random Points. A previous version of this paper used the name DRP ("Distance to Random Point").} credible region for a given $d$ and $\theta_r$: 
     \revised{
    \begin{equation}
        \left\{ \theta \in U \ | \ d(\theta, \theta_r) \leq d(\theta^*, \theta_r) \right\}
    \end{equation}
    }
    This, in turn, defines a corresponding  {\bf TARP confidence level} 
    $1 - \tilde{\alpha}_{\rm TARP}(\hat{p}, \theta^*, \theta_r, d)$. \revised{as the integral of $\hat{p}(\theta | x)$ over that region}.
\end{remark}

% As was the case for HPD regions, we could use~\cref{eq:name_ecp} and ~\cref{eq:ecp} to calculate ECPs directly. However, calculating the DRP credible region for a given $\alpha$, requires finding out the radius $R$ from~\cref{eq:name_ecp}, which is non-trivial. Therefore, we propose using an alternative definition, similar to what was done in~\cref{sec:prev_work}. 
We can calculate expected coverage similar to the HPD case:
\begin{lemma}
    We can calculate the ECP of the $1 - \alpha$ TARP regions as:
    \begin{equation}\label{eq:ECP_drp_alt}
        \ecp(\hat{p}, \alpha, \mathcal{D}_{\theta_r}) = \mathbb{E}_{p(\theta, x)} \left[ \mathds{1}(\tilde{\alpha}_{\rm TARP}(\hat{p}, \theta^*, \theta_r, x^*, d) \geq \alpha) \right].
    \end{equation}
\end{lemma}

\begin{proof}
    Let $\mathcal{D}_{\theta_r}\left( x^*, \alpha, \hat{p}, d \right)$ be a ball centered at $\theta_r$ with radius $R(\hat{p}, \alpha, x)$ and credibility $1 - \alpha$. Similarly, the TARP region defined by $(\theta^*, x^*)$ has the same center, radius $d(\theta^*, \theta_r)$, and credibility $1 - \tilde{\alpha}$ for some $\tilde{\alpha}$. It then follows that: 
    \begin{equation}
        \theta^* \in \mathcal{D}_{\theta_r}\left( x^*, \alpha, \hat{p}, d \right) \Leftrightarrow d(\theta^*, \theta_r) \leq R(\hat{p}, \alpha, x).
    \end{equation}
    Since $R$ is a monotonic function of $\alpha$ and the regions are centered on the same point, we have
    \begin{equation}
        d(\theta^*, \theta_r) \leq R(\hat{p}, \alpha, x) \Leftrightarrow \tilde{\alpha} \geq \alpha \, .
    \end{equation}
    Then by \cref{eq:ecp} we have \cref{eq:ECP_drp_alt}.
\end{proof}

With this, we have everything we need to formulate our algorithm, which is presented in~\cref{alg:ecp_name}. While similar to~\cref{alg:ecp_hpd}, there are \revised{three} key differences to this algorithm: 
\begin{itemize}
    \item TARP implements \cref{theorem:main}'s requirement that coverage holds for all possible ways of choosing the positions of the credible regions by randomly sampling $\theta_r$ from some distribution $\tilde{p}(\theta | x)$ that can depend on $x$.
    \item TARP probes credible regions of smaller size (i.e., larger $\alpha$) as the number of posterior samples, simulations, and reference points tested is increased. Following the logic of the proof of \cref{theorem:main}, this means it asymptotically tests whether the averages of $\hat{p}(\theta | x)$ and $p(\theta | x)$ match on smaller and small balls.
    \item TARP does not require explicit evaluations of the posterior estimator $\hat{p}$: it only requires calculating distances between parameters sampled from $\hat{p}$ and $\theta_r$.
\end{itemize}

 In the following section, we test the proposed method in a series of experiments and compare its performance with that of HPD coverage probabilities.

\section{Experiments}\label{sec:experiments}

We apply our algorithm, described in~\cref{alg:ecp_name} to three different experiments. For all experiments, we normalize all parameters $\theta$ to the range $[0, 1]$, and unless otherwise specified, we generate reference points uniformly in the $D$-dimensional hypercube $x \in [0, 1]^D$ where $D$ is the dimensionality of the parameter space. We use the Euclidean or L2 distance as a metric to calculate TARP regions. \revised{We explore the dependence on the reference point distribution and the distance metric in~\cref{sec:dependence}.}

\subsection{Gaussian Toy Model}\label{sec:toy}

As a first example, we can use a simple Gaussian toy model. In this model, we assume that all the posterior distributions are Gaussian. Therefore, we can generate samples from the posterior for a validation simulation from the estimated mean and covariance matrix. 
We first generate `simulations', by uniform sampling in our parameter space, $\theta^* \sim \mathcal{U}(-5, 5)$. We also randomly generate the diagonal elements of the covariances matrices $\Sigma$ of our posterior estimates by sampling from $\log \sigma \sim \mathcal{U}(-5, -1)$, and set the off-diagonal elements to $0$. To validate, we also need to know the mean of the posterior distributions. We consider three cases: 
\begin{itemize}
    \item Firstly, we draw these from a normal distribution \revised{$\mathcal{N}(\theta^*, \Sigma)$}. This means that the coverage probabilities should show a uniform distribution. We call this the {\it correct case}.
    \item Secondly, we draw the true values from \revised{$\mathcal{N}(\theta^*, 0.5 \Sigma)$ and $\left(\mathcal{N}(\theta^*, 2 \Sigma) \right)$}. This means that the posterior samples come from a distribution that is too narrow (wide), and are therefore overconfident (underconfident)
    \item Lastly, we want to build a {\it biased case}. For this, we pick the means to be equal to: 
    \revised{
    \begin{equation}\label{eq:biased}
        \theta^* - {\rm sign}(\theta^*) \cdot Z \left( 1 - \frac{|\theta^*|}{5}\right) \cdot \sigma,
    \end{equation}
    }
    where $Z$ is the inverse survival function. The idea with this example is to create a position-dependent bias: The furthest the true value is from the origin, the more biased the posterior is. We have specifically designed this bias in a way that \revised{HPD} coverage probabilities will be blind to it. However, the point of this example is to show that there are biases that \revised{HPD} can be blind to, but the random nature of TARP should be able to detect. \revised{The function~\cref{eq:biased} is plotted in appendix~\cref{app:biased_case}}
\end{itemize}
For each of these  cases, we want to compare how this method compares to the HPD coverage probability test. Because in this toy model we know the correct posterior, we can easily compute both HPD and TARP coverage probabilities. To pick the TARP reference points, we use the prior ($\tilde{p}(\theta_r | x) = p(\theta_r)$).

The results for our Gaussian toy model are shown in~\cref{fig:toy_gaussian}. In each panel, the $x$-axis shows the credibility level $1 - \alpha$, while the $y$-axis shows the expected coverage ${\rm ECP}(\hat{p}, \alpha, \mathcal{G})$. For an accurate posterior estimator, ${\rm ECP} (\hat{p}, \alpha\mathcal{G}) = 1 - \alpha,  \ \forall \alpha \in (0, 1)$ as described in~\cref{sec:formalism}, which would then lead to the diagonal black dashed diagonal line. We see in the first panel that that is indeed the case for the `correct' case, which is accurate by construction. We found consistent results amongst all values of $D$ we tested, going up to $D = 1000$.

The second and third panels show the over and underconfident cases, respectively. We see how these cases lead to different coverage plots than the HPD method. This is not entirely unexpected: 
For underconfident estimators, the TARP regions from randomly selected points are more likely to cover approximately half of the posterior estimator $\alpha \sim 0.5$, while for overconfident estimators, they are likely to cover either very little $\alpha \sim 1$ or a lot $\alpha \sim 0$.
We expand this intuition, including some figures, in~\cref{app:intuition}. Finally, in the fourth panel, we see how the biased case cannot be detected by the HPD region but is detectable by TARP. This shows how, as explained in~\cref{sec:formalism}, ${\rm ECP} (\hat{p}, \alpha) = 1 - \alpha$ does not mean the posterior is accurate for HPD regions, but it does for TARP regions.  
%expected, our method not only scales better with dimensionality, but is also robust to certain biases, that the HPD method can be blind to. 

\revised{We also repeated this example for the case of Gaussian distributions with nondiagonal elements in the covariance matrix. To do this, we randomly generated arrays of size $D (D - 1)/2$, we then converted them into lower triangular matrices, which we used as the Cholesky decomposition of the covariance matrix. We found that adding non-zero elements to the covariance matrix did not change our results.}

\subsection{\revised{Dependence on $\theta_r$ distribution and distance metric}}
\label{sec:dependence}

\revised{All the results of the Gaussian Toy Model, shown in~\cref{fig:toy_gaussian}, rely on two choices, specified in~\cref{sec:drp_method}: A distribution to draw reference points $\theta_r$ from, and a distance metric 
 $d(\cdot, \cdot)$. Therefore, it is key to study the dependence of our method on different choices of both things. 
}
\revised{
Firstly, we repeated all four versions of the Gaussian Toy Model experiment, drawing $\theta_r$ from various distributions: 
\begin{itemize}
\item A uniform distribution, both covering the wide range $\theta_r \sim \mathcal{U}(0,1)$, and covering only part of the range $\theta_r \sim \mathcal{U}(0,0.5)$
\item A normal distribution, centered at $\theta = 0.5$, and with standard deviation varying between $0.01$ and $0.1$
\item $\theta_r$ with a fixed value, either at the center of parameter space $\theta_r = 0.5$ or at a different location.
\end{itemize}
}

\revised{
We also repeated our experiments using the Manhattan or L1 distance, instead of L2. We found very similar curves to those shown in~\cref{fig:toy_gaussian} for the correct, overconfident, and underconfident cases. In the biased case, the different $\theta_r$ distributions led to different curves, but all of them clearly showed there was a bias. These figures are shown in~\cref{app:dependence}. Therefore, we conclude that the proposed method is robust to different distributions for $\theta_r$, and choices of distance metric.
}

\subsection{Revealing when estimators are uninformative}\label{sec:posterior-prior}

% \amc{make sure this ties in}

\begin{figure}[t]%[tbhp]
    \centering
    \includegraphics[width=.89\linewidth]{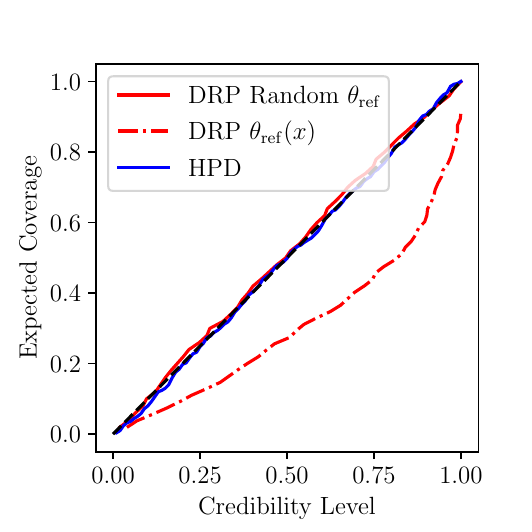}
    \caption{
    Expected coverage vs credibility level for the uninformative posterior estimator described in~\cref{sec:posterior-prior}. The blue line shows the coverage calculated using HPD regions, while the red lines use TARP regions. The continuous line uses reference points that are independent of $x$, while the dot-dashed line uses reference points that depend on $x$.
        \label{fig:uninformative}
    }
\end{figure}

As our second benchmark, we consider the case mentioned before in which the learned posterior estimator is equal to the prior $\hat{p}(\theta | x) = p(\theta)$. The reason why we are interested in this example is that, in that case, the expected coverage probability calculated using HPD regions will be equal to $1 - \alpha$ for any value of $\alpha$, as previously discussed. However, with TARP we have the ability to avoid this blindspot by sampling reference points in a manner dependent on $x$.

To make this concrete, we consider a one-dimensional example with a Gaussian prior $p (\theta) = \mathcal{N}(\theta; \mu_0, \sigma_0^2)$. Our `forward model' in this case is simply generating a number $n_x$ of data points, from $\left\{ x_i \right\}_{i = 1}^{n_x} \sim N(\theta, \revised{\sigma_x}^2)$. In this conjugate model, we can easily derive the true posterior:
\begin{align}
    p\left (\theta | \left\{ x_i \right\}_{i = 1}^{n_x} \right) &= \mathcal{N} (\mu | m, s), \\
s &= \left( {1 \over \sigma_0^2} + {n+x \over \revised{\sigma_x}+x^2} \right)^{-1}, \\ 
m &= s \left( {\mu_0 \over \sigma_0^2} + {\sum_i x_i \over \sigma_x^2} \right)
\end{align}
We fix $n_x = 50$, $\mu_0 = 0$, $\sigma_0 = 1$ and $\sigma_x = 0.1$. We generate $500$ samples from the forward model, and calculate expected coverage from an `uninformative estimator' $\hat{p}(\theta | x) = p(\theta)$ in three ways: 1) using HPD regions, 2) using TARP regions where $\theta_r$ is drawn randomly from \revised{$\mathcal{U}(0, 1)$}, and 3) using TARP regions where $\theta_r = x_0 + u$, where $x_0$ is the first observation, and $u \sim \mathcal{U}(0, 1)$. We expect the first two methods to have ECP equal to $1 - \alpha$, but not for the third.

We show the results in~\cref{fig:uninformative}. First, we notice that when we use HPD regions, we get the correct expected coverage, even though the estimator is wrong (validating the theoretical discussion in \cref{sec:formalism}). This means that, in this case, HPD coverage could fool us into thinking our estimator is accurate when in reality it is completely uninformative. Interestingly, the same happens when we use TARP regions with reference points selected randomly from the prior (red line). This is because, as discussed in~\cref{sec:formalism}, \cref{theorem:main} only holds in both directions when the choice of the region depends on $x$. Finally, as anticipated, the expected coverage is \emph{not} $1 - \alpha$ when the sampling distribution for $\theta_r$ has some $x$-dependence. Therefore, we see how even when we introduce a small dependence on $x$ to $\tilde{p}(\theta_r | x)$ in TARP reveals that the posterior estimator is not accurate. 
\revised{We further explore how the dependence of $\tilde{p}(\theta_r | x)$ on x affects our results in~\cref{app:dependence}.
}

\subsection{Gravitational Lensing}\label{sec:gl}

\begin{figure}[t]
    \centering
    \includegraphics[width=.89\linewidth]{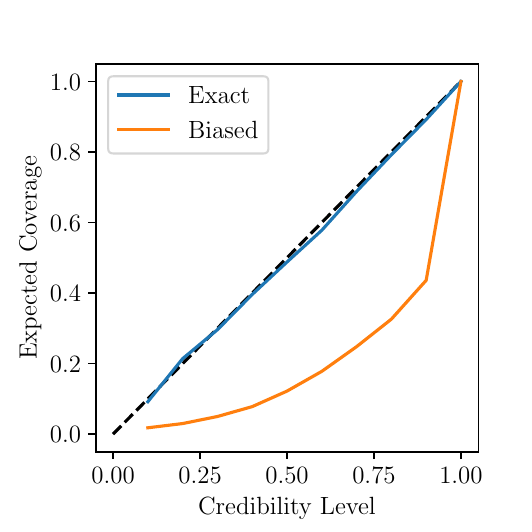}
    \caption{Expected coverage probability vs credibility level for our lensing example\revised{, for which tests based on HPD coverage are intractable}. We see how, as expected, the exact \revised{posterior} estimator (blue) \revised{accurately characterizes the posterior while the biased estimator (orange) does not}.}
    \label{fig:lensing_cov}
\end{figure}

To test our algorithm in a more realistic and high-dimensional setting, we consider a simplified astrophysics problem: gravitational lensing source reconstruction. Gravitational lensing occurs in nature when light rays from a distant galaxy move along curved rather than straight paths due to the mass of another intervening galaxy (the `lens') \citep{treu2010strong}. The result is a highly-distorted, ring-shaped image of the background galaxy. The goal of source reconstruction is to infer from a noisy image what the light from the source galaxy looks like without distortions, assuming the mathematical form of the distortions is perfectly known. \revised{In this high-dimensional setting, coverage checks based on the posterior's HPD region are intractable.} 

The simulator in this scenario samples the source galaxy's light $\theta$ from a multivariate-normal distribution that we fit to a dataset of galaxy images \citep{2019ApJ...882....6S, 2021ApJ...912...41S}. A matrix $A$ encoding the lensing distortions are then applied, and the final observation is produced by adding Gaussian pixel noise of standard deviation $\sigma_n$, so that $x \sim \mathcal{N}(A \theta, \sigma_n^2)$. For computational convenience, we use $16 \times 16$-pixel source images and $32 \times 32$-pixel observations.

As shown in \citet{adam2022posterior} and reviewed in \cref{app:lensing}, posterior samples of $\theta$ can be generated using techniques from diffusion modeling. In general, this approach yields subtly biased posterior samples. However, with our multivariate-normal prior on $\theta$, it is possible to generate unbiased posterior samples. We refer to samples from these as `biased' and `exact' in our results.

\cref{fig:lensing_cov} shows the results for both the exact and the biased posterior estimators, using \revised{$500$} simulations, and $1000$ posterior samples per simulation, and sampling $\theta_r$ from the prior. As expected, our method gets the correct coverage for the exact estimator. It is important to stress that generative models are needed for parameter spaces of this dimensionality (256 parameters), and no previously existing methods could calculate ECPs to assess the accuracy of such models. 
The biased estimator, on the other hand, produces a similar curve to that of the bottom right panel of~\cref{fig:toy_gaussian}, which indicates that it is indeed biased.

% To gain a better understanding of what is going on with both estimators, we can repeat the coverage calculation in each pixel separately. We can then calculate MSE between the coverage and credibility level for each pixel. The results of doing this are shown in~\cref{fig:mse_cov}. We see that the biased estimator produces particularly bad estimates in certain pixels on the left side of the image, which are likely to be driving the bias in the full-dimensional space. It is important to point out the importance of testing in the full-dimensional space: 
% %As shown by the toy example of~\cref{fig:toy}, 
% When calculating coverage diagnostics on marginalized posteriors, we loose information about potential complex degeneracies between parameters. While this pixel by pixel test is good for interpretability, it is certainly not sufficient diagnostic. 

% \section{Results}\label{sec:results}
% %%%%
% \amc{A few questions: 1) does the sampling distribution for the random points matter? I can imagine the shapes of the pp curves would change for different choices. 2) Have you thought about how to get MC error estimates for the pp curves? Might not be too hard to modify the procedure people typically use for this (Jeffreys intervals since you're doing binomial trials).}

% The reason the lensing case is biased (not overconfident) is that it gets smaller sigma in the middle of the image, and bigger everywhere else, but the middle is also where the signal (truth) is higher, therefore it is overconfident for some values of the truth, and underconfident for others (i.e. biased). 

\section{Conclusions}\label{sec:conc}

Testing the accuracy of estimated posteriors is a key element of parameter inference. While there exist well-establish \revised{convergence diagnostics for established sampling methods like MCMC}, it is difficult to directly test the accuracy of posterior \revised{inferences, particularly those computed using deep learning methods}. This is the case for both likelihood-based and simulation-based inference. In this paper, we introduced TARP coverage probabilities as a new technique to test the accuracy of estimated posteriors using posterior samples alone, when explicit posterior evaluations are not available. \revised{While our focus is testing posterior estimators based on generative machine learning models, our method could equally well be used to test the correctness of MCMC samples, although potentially at a great computational cost.}

We have shown that this test is sufficient to prove that the inference is accurate, while other similar tests were necessary but not sufficient. \revised{We also tested the impact of the choice of $\hat{p}(\theta_r | x)$ and the distance metric used by the TARP method and found that they do not significantly affect our results. The exception to this is the case where the posterior estimator is equal to the prior, in which case TARP only works if $\hat{p}(\theta_r | x)$ has some dependency on $x$. It is left up to the user of the method to determine whether this is a risk.} 

We applied our test successfully to a variety of inference problems, in particular in cases where alternative methods fail, and showed that it scales well to high-dimensional posteriors. Therefore, we propose TARP coverage probabilities as a tool to test the accuracy of future posterior inference analyses from generative models. 

\section{\revised{Broader Impact}}

\revised{
Our work is focused on checking the correctness of statistical inferences, which is an important open issue. We expect our work to have a positive societal impact by increasing the trustworthiness of machine learning applications to scientific problems across a wide variety of domains. As with any statistical method, however, incorrect application of our method (particularly through poor choice of the sampling distribution for $\theta_\mathrm{ref}$) could lead to invalid conclusions.
}

\bibliographystyle{style/icml2023}
\bibliography{ref}

\begin{thebibliography}{72}
\providecommand{\natexlab}[1]{#1}
\providecommand{\url}[1]{\texttt{#1}}
\expandafter\ifx\csname urlstyle\endcsname\relax
  \providecommand{\doi}[1]{doi: #1}\else
  \providecommand{\doi}{doi: \begingroup \urlstyle{rm}\Url}\fi

\bibitem[Adam et~al.(2022)Adam, Coogan, Malkin, Legin, Perreault-Levasseur,
  Hezaveh, and Bengio]{adam2022posterior}
Adam, A., Coogan, A., Malkin, N., Legin, R., Perreault-Levasseur, L., Hezaveh,
  Y., and Bengio, Y.
\newblock Posterior samples of source galaxies in strong gravitational lenses
  with score-based priors.
\newblock \emph{arXiv preprint arXiv:2211.03812}, 2022.

\bibitem[Alsing et~al.(2019)Alsing, Charnock, Feeney, and
  Wandelt]{alsing2019fast}
Alsing, J., Charnock, T., Feeney, S., and Wandelt, B.
\newblock Fast likelihood-free cosmology with neural density estimators and
  active learning.
\newblock \emph{Monthly Notices of the Royal Astronomical Society},
  488\penalty0 (3):\penalty0 4440--4458, 2019.

\bibitem[Anau~Montel \& Weniger(2022)Anau~Montel and
  Weniger]{AnauMontel:2022ppb}
Anau~Montel, N. and Weniger, C.
\newblock {Detection is truncation: studying source populations with truncated
  marginal neural ratio estimation}.
\newblock In \emph{{36th Conference on Neural Information Processing Systems}},
  11 2022.

\bibitem[Beaumont et~al.(2002)Beaumont, Zhang, and
  Balding]{beaumont2002approximate}
Beaumont, M.~A., Zhang, W., and Balding, D.~J.
\newblock Approximate bayesian computation in population genetics.
\newblock \emph{Genetics}, 162\penalty0 (4):\penalty0 2025--2035, 2002.

\bibitem[{Brehmer}(2021)]{2021NatRP...3..305B}
{Brehmer}, J.
\newblock {Simulation-based inference in particle physics}.
\newblock \emph{Nature Reviews Physics}, 3\penalty0 (5):\penalty0 305--305,
  January 2021.
\newblock \doi{10.1038/s42254-021-00305-6}.

\bibitem[Brehmer et~al.(2019)Brehmer, Mishra-Sharma, Hermans, Louppe, and
  Cranmer]{Brehmer:2019jyt}
Brehmer, J., Mishra-Sharma, S., Hermans, J., Louppe, G., and Cranmer, K.
\newblock {Mining for Dark Matter Substructure: Inferring subhalo population
  properties from strong lenses with machine learning}.
\newblock \emph{Astrophys. J.}, 886\penalty0 (1):\penalty0 49, 2019.
\newblock \doi{10.3847/1538-4357/ab4c41}.

\bibitem[{Charnock} et~al.(2020){Charnock}, {Perreault-Levasseur}, and
  {Lanusse}]{2020arXiv200601490C}
{Charnock}, T., {Perreault-Levasseur}, L., and {Lanusse}, F.
\newblock {Bayesian Neural Networks}.
\newblock \emph{arXiv e-prints}, art. arXiv:2006.01490, June 2020.
\newblock \doi{10.48550/arXiv.2006.01490}.

\bibitem[Chen et~al.(2020)Chen, Zhang, Gutmann, Courville, and
  Zhu]{chen2020neural}
Chen, Y., Zhang, D., Gutmann, M., Courville, A., and Zhu, Z.
\newblock Neural approximate sufficient statistics for implicit models.
\newblock \emph{arXiv preprint arXiv:2010.10079}, 2020.

\bibitem[Coogan et~al.(2020)Coogan, Karchev, and Weniger]{Coogan:2020yux}
Coogan, A., Karchev, K., and Weniger, C.
\newblock {Targeted Likelihood-Free Inference of Dark Matter Substructure in
  Strongly-Lensed Galaxies}.
\newblock In \emph{{34th Conference on Neural Information Processing Systems}},
  10 2020.

\bibitem[Coogan et~al.(2022)Coogan, Anau~Montel, Karchev, Grootes, Nattino, and
  Weniger]{Coogan:2022cky}
Coogan, A., Anau~Montel, N., Karchev, K., Grootes, M.~W., Nattino, F., and
  Weniger, C.
\newblock {One never walks alone: the effect of the perturber population on
  subhalo measurements in strong gravitational lenses}.
\newblock 9 2022.

\bibitem[Cranmer et~al.(2015)Cranmer, Pavez, and
  Louppe]{cranmer2015approximating}
Cranmer, K., Pavez, J., and Louppe, G.
\newblock Approximating likelihood ratios with calibrated discriminative
  classifiers.
\newblock \emph{arXiv preprint arXiv:1506.02169}, 2015.

\bibitem[Cranmer et~al.(2020)Cranmer, Brehmer, and Louppe]{cranmer2020frontier}
Cranmer, K., Brehmer, J., and Louppe, G.
\newblock The frontier of simulation-based inference.
\newblock \emph{Proceedings of the National Academy of Sciences}, 117\penalty0
  (48):\penalty0 30055--30062, 2020.

\bibitem[Dalmasso et~al.(2020)Dalmasso, Pospisil, Lee, Izbicki, Freeman, and
  Malz]{dalmasso2020conditional}
Dalmasso, N., Pospisil, T., Lee, A.~B., Izbicki, R., Freeman, P.~E., and Malz,
  A.~I.
\newblock Conditional density estimation tools in python and r with
  applications to photometric redshifts and likelihood-free cosmological
  inference.
\newblock \emph{Astronomy and Computing}, 30:\penalty0 100362, 2020.

\bibitem[Dax et~al.(2021)Dax, Green, Gair, Macke, Buonanno, and
  Sch\"olkopf]{PhysRevLett.127.241103}
Dax, M., Green, S.~R., Gair, J., Macke, J.~H., Buonanno, A., and Sch\"olkopf,
  B.
\newblock Real-time gravitational wave science with neural posterior
  estimation.
\newblock \emph{Phys. Rev. Lett.}, 127:\penalty0 241103, Dec 2021.
\newblock \doi{10.1103/PhysRevLett.127.241103}.
\newblock URL \url{https://link.aps.org/doi/10.1103/PhysRevLett.127.241103}.

\bibitem[de~Witt et~al.(2020)de~Witt, Gram-Hansen, Nardelli, Gambardella,
  Zinkov, Dokania, Siddharth, Espinosa-Gonzalez, Darzi, Torr, and
  Baydin]{https://doi.org/10.48550/arxiv.2005.07062}
de~Witt, C.~S., Gram-Hansen, B., Nardelli, N., Gambardella, A., Zinkov, R.,
  Dokania, P., Siddharth, N., Espinosa-Gonzalez, A.~B., Darzi, A., Torr, P.,
  and Baydin, A.~G.
\newblock Simulation-based inference for global health decisions.
\newblock 2020.
\newblock \doi{10.48550/ARXIV.2005.07062}.
\newblock URL \url{https://arxiv.org/abs/2005.07062}.

\bibitem[Deistler et~al.(2022)Deistler, Goncalves, and
  Macke]{deistler2022truncated}
Deistler, M., Goncalves, P.~J., and Macke, J.~H.
\newblock Truncated proposals for scalable and hassle-free simulation-based
  inference.
\newblock \emph{arXiv preprint arXiv:2210.04815}, 2022.

\bibitem[Dinh et~al.(2014)Dinh, Krueger, and Bengio]{dinh2014nice}
Dinh, L., Krueger, D., and Bengio, Y.
\newblock Nice: Non-linear independent components estimation.
\newblock \emph{arXiv preprint arXiv:1410.8516}, 2014.

\bibitem[Durkan et~al.(2020)Durkan, Murray, and
  Papamakarios]{durkan2020contrastive}
Durkan, C., Murray, I., and Papamakarios, G.
\newblock On contrastive learning for likelihood-free inference.
\newblock In \emph{International Conference on Machine Learning}, pp.\
  2771--2781. PMLR, 2020.

\bibitem[Fearnhead \& Prangle(2012)Fearnhead and
  Prangle]{fearnhead2012constructing}
Fearnhead, P. and Prangle, D.
\newblock Constructing summary statistics for approximate bayesian computation:
  semi-automatic approximate bayesian computation.
\newblock \emph{Journal of the Royal Statistical Society: Series B (Statistical
  Methodology)}, 74\penalty0 (3):\penalty0 419--474, 2012.

\bibitem[Frazier et~al.(2022)Frazier, Nott, Drovandi, and
  Kohn]{frazier2022bayesian}
Frazier, D.~T., Nott, D.~J., Drovandi, C., and Kohn, R.
\newblock Bayesian inference using synthetic likelihood: asymptotics and
  adjustments.
\newblock \emph{Journal of the American Statistical Association}, \penalty0
  (just-accepted):\penalty0 1--28, 2022.

\bibitem[Gelman \& Rubin(1992)Gelman and Rubin]{gelman1992inference}
Gelman, A. and Rubin, D.~B.
\newblock Inference from iterative simulation using multiple sequences.
\newblock \emph{Statistical science}, pp.\  457--472, 1992.

\bibitem[Gonçalves et~al.(2020)Gonçalves, Lueckmann, Deistler, Nonnenmacher,
  Öcal, Bassetto, Chintaluri, Podlaski, Haddad, Vogels, Greenberg, and
  Macke]{10.7554/eLife.56261}
Gonçalves, P.~J., Lueckmann, J.-M., Deistler, M., Nonnenmacher, M., Öcal, K.,
  Bassetto, G., Chintaluri, C., Podlaski, W.~F., Haddad, S.~A., Vogels, T.~P.,
  Greenberg, D.~S., and Macke, J.~H.
\newblock Training deep neural density estimators to identify mechanistic
  models of neural dynamics.
\newblock \emph{eLife}, 9:\penalty0 e56261, sep 2020.
\newblock ISSN 2050-084X.
\newblock \doi{10.7554/eLife.56261}.
\newblock URL \url{https://doi.org/10.7554/eLife.56261}.

\bibitem[Goodfellow et~al.(2014)Goodfellow, Pouget-Abadie, Mirza, Xu,
  Warde-Farley, Ozair, Courville, and Bengio]{goodfellow2014generative}
Goodfellow, I.~J., Pouget-Abadie, J., Mirza, M., Xu, B., Warde-Farley, D.,
  Ozair, S., Courville, A., and Bengio, Y.
\newblock Generative adversarial nets.
\newblock \emph{stat}, 1050:\penalty0 10, 2014.

\bibitem[Greenberg et~al.(2019)Greenberg, Nonnenmacher, and
  Macke]{greenberg2019automatic}
Greenberg, D., Nonnenmacher, M., and Macke, J.
\newblock Automatic posterior transformation for likelihood-free inference.
\newblock In \emph{International Conference on Machine Learning}, pp.\
  2404--2414. PMLR, 2019.

\bibitem[{Guo} et~al.(2017){Guo}, {Pleiss}, {Sun}, and
  {Weinberger}]{2017arXiv170604599G}
{Guo}, C., {Pleiss}, G., {Sun}, Y., and {Weinberger}, K.~Q.
\newblock {On Calibration of Modern Neural Networks}.
\newblock \emph{arXiv e-prints}, art. arXiv:1706.04599, June 2017.
\newblock \doi{10.48550/arXiv.1706.04599}.

\bibitem[Hermans et~al.(2020)Hermans, Begy, and Louppe]{hermans2020likelihood}
Hermans, J., Begy, V., and Louppe, G.
\newblock Likelihood-free mcmc with amortized approximate ratio estimators.
\newblock In \emph{International Conference on Machine Learning}, pp.\
  4239--4248. PMLR, 2020.

\bibitem[Hermans et~al.(2021{\natexlab{a}})Hermans, Banik, Weniger, Bertone,
  and Louppe]{Hermans:2020skz}
Hermans, J., Banik, N., Weniger, C., Bertone, G., and Louppe, G.
\newblock {Towards constraining warm dark matter with stellar streams through
  neural simulation-based inference}.
\newblock \emph{Mon. Not. Roy. Astron. Soc.}, 507\penalty0 (2):\penalty0
  1999--2011, 2021{\natexlab{a}}.
\newblock \doi{10.1093/mnras/stab2181}.

\bibitem[Hermans et~al.(2021{\natexlab{b}})Hermans, Delaunoy, Rozet, Wehenkel,
  and Louppe]{hermans2021averting}
Hermans, J., Delaunoy, A., Rozet, F., Wehenkel, A., and Louppe, G.
\newblock Averting a crisis in simulation-based inference.
\newblock \emph{arXiv preprint arXiv:2110.06581}, 2021{\natexlab{b}}.

\bibitem[Ho et~al.(2020)Ho, Jain, and
  Abbeel]{DBLP:journals/corr/abs-2006-11239}
Ho, J., Jain, A., and Abbeel, P.
\newblock Denoising diffusion probabilistic models.
\newblock \emph{CoRR}, abs/2006.11239, 2020.
\newblock URL \url{https://arxiv.org/abs/2006.11239}.

\bibitem[Hyv{\"a}rinen \& Dayan(2005)Hyv{\"a}rinen and
  Dayan]{hyvarinen2005estimation}
Hyv{\"a}rinen, A. and Dayan, P.
\newblock Estimation of non-normalized statistical models by score matching.
\newblock \emph{Journal of Machine Learning Research}, 6\penalty0 (4), 2005.

\bibitem[Karchev et~al.(2022{\natexlab{a}})Karchev, Anau~Montel, Coogan, and
  Weniger]{Karchev:2022ycy}
Karchev, K., Anau~Montel, N., Coogan, A., and Weniger, C.
\newblock {Strong-Lensing Source Reconstruction with Denoising Diffusion
  Restoration Models}.
\newblock In \emph{{36th Conference on Neural Information Processing Systems}},
  11 2022{\natexlab{a}}.

\bibitem[Karchev et~al.(2022{\natexlab{b}})Karchev, Trotta, and
  Weniger]{Karchev:2022xyn}
Karchev, K., Trotta, R., and Weniger, C.
\newblock {SICRET: Supernova Ia Cosmology with truncated marginal neural Ratio
  EsTimation}.
\newblock 9 2022{\natexlab{b}}.
\newblock \doi{10.1093/mnras/stac3785}.

\bibitem[Kingma \& Welling(2013)Kingma and Welling]{kingma2013auto}
Kingma, D.~P. and Welling, M.
\newblock Auto-encoding variational bayes.
\newblock \emph{arXiv preprint arXiv:1312.6114}, 2013.

\bibitem[Legin et~al.(2021)Legin, Hezaveh, Levasseur, and
  Wandelt]{Legin:2021zup}
Legin, R., Hezaveh, Y., Levasseur, L.~P., and Wandelt, B.
\newblock {Simulation-Based Inference of Strong Gravitational Lensing
  Parameters}.
\newblock 12 2021.

\bibitem[Linhart et~al.(2022)Linhart, Gramfort, and
  Rodrigues]{linhart2022validation}
Linhart, J., Gramfort, A., and Rodrigues, P.~L.
\newblock Validation diagnostics for sbi algorithms based on normalizing flows.
\newblock \emph{arXiv preprint arXiv:2211.09602}, 2022.

\bibitem[Lueckmann et~al.(2017)Lueckmann, Goncalves, Bassetto, {\"O}cal,
  Nonnenmacher, and Macke]{lueckmann2017flexible}
Lueckmann, J.-M., Goncalves, P.~J., Bassetto, G., {\"O}cal, K., Nonnenmacher,
  M., and Macke, J.~H.
\newblock Flexible statistical inference for mechanistic models of neural
  dynamics.
\newblock \emph{Advances in neural information processing systems}, 30, 2017.

\bibitem[{Lueckmann} et~al.(2018){Lueckmann}, {Bassetto}, {Karaletsos}, and
  {Macke}]{2018arXiv180509294L}
{Lueckmann}, J.-M., {Bassetto}, G., {Karaletsos}, T., and {Macke}, J.~H.
\newblock {Likelihood-free inference with emulator networks}.
\newblock \emph{arXiv e-prints}, art. arXiv:1805.09294, May 2018.

\bibitem[Lueckmann et~al.(2021)Lueckmann, Boelts, Greenberg, Goncalves, and
  Macke]{lueckmann2021benchmarking}
Lueckmann, J.-M., Boelts, J., Greenberg, D., Goncalves, P., and Macke, J.
\newblock Benchmarking simulation-based inference.
\newblock In \emph{International Conference on Artificial Intelligence and
  Statistics}, pp.\  343--351. PMLR, 2021.

\bibitem[Marjoram et~al.(2003)Marjoram, Molitor, Plagnol, and
  Tavar{\'e}]{marjoram2003markov}
Marjoram, P., Molitor, J., Plagnol, V., and Tavar{\'e}, S.
\newblock Markov chain monte carlo without likelihoods.
\newblock \emph{Proceedings of the National Academy of Sciences}, 100\penalty0
  (26):\penalty0 15324--15328, 2003.

\bibitem[Marlier et~al.(2021)Marlier, Brüls, and
  Louppe]{https://doi.org/10.48550/arxiv.2109.14275}
Marlier, N., Brüls, O., and Louppe, G.
\newblock Simulation-based bayesian inference for multi-fingered robotic
  grasping, 2021.
\newblock URL \url{https://arxiv.org/abs/2109.14275}.

\bibitem[Miller et~al.(2022{\natexlab{a}})Miller, Cole, Weniger, Nattino, Ku,
  and Grootes]{Miller2022}
Miller, B.~K., Cole, A., Weniger, C., Nattino, F., Ku, O., and Grootes, M.~W.
\newblock swyft: Truncated marginal neural ratio estimation in python.
\newblock \emph{Journal of Open Source Software}, 7\penalty0 (75):\penalty0
  4205, 2022{\natexlab{a}}.
\newblock \doi{10.21105/joss.04205}.
\newblock URL \url{https://doi.org/10.21105/joss.04205}.

\bibitem[Miller et~al.(2022{\natexlab{b}})Miller, Weniger, and
  Forr\'e]{Miller:2022haf}
Miller, B.~K., Weniger, C., and Forr\'e, P.
\newblock {Contrastive Neural Ratio Estimation}.
\newblock 10 2022{\natexlab{b}}.

\bibitem[Mishra-Sharma \& Cranmer(2022)Mishra-Sharma and
  Cranmer]{Mishra-Sharma:2021oxe}
Mishra-Sharma, S. and Cranmer, K.
\newblock {Neural simulation-based inference approach for characterizing the
  Galactic Center \ensuremath{\gamma}-ray excess}.
\newblock \emph{Phys. Rev. D}, 105\penalty0 (6):\penalty0 063017, 2022.
\newblock \doi{10.1103/PhysRevD.105.063017}.

\bibitem[Montel et~al.(2022)Montel, Coogan, Correa, Karchev, and
  Weniger]{Montel:2022fhv}
Montel, N.~A., Coogan, A., Correa, C., Karchev, K., and Weniger, C.
\newblock {Estimating the warm dark matter mass from strong lensing images with
  truncated marginal neural ratio estimation}.
\newblock \emph{Mon. Not. Roy. Astron. Soc.}, 518\penalty0 (2):\penalty0
  2746--2760, 2022.
\newblock \doi{10.1093/mnras/stac3215}.

\bibitem[Ong et~al.(2018)Ong, Nott, Tran, Sisson, and
  Drovandi]{ong2018likelihood}
Ong, V. M.-H., Nott, D.~J., Tran, M.-N., Sisson, S.~A., and Drovandi, C.~C.
\newblock Likelihood-free inference in high dimensions with synthetic
  likelihood.
\newblock \emph{Computational Statistics \& Data Analysis}, 128:\penalty0
  271--291, 2018.

\bibitem[Papamakarios \& Murray(2016)Papamakarios and
  Murray]{papamakarios2016fast}
Papamakarios, G. and Murray, I.
\newblock Fast $\varepsilon$-free inference of simulation models with bayesian
  conditional density estimation.
\newblock \emph{Advances in neural information processing systems}, 29, 2016.

\bibitem[Papamakarios et~al.(2019)Papamakarios, Sterratt, and
  Murray]{papamakarios2019sequential}
Papamakarios, G., Sterratt, D., and Murray, I.
\newblock Sequential neural likelihood: Fast likelihood-free inference with
  autoregressive flows.
\newblock In \emph{The 22nd International Conference on Artificial Intelligence
  and Statistics}, pp.\  837--848. PMLR, 2019.

\bibitem[Papamakarios et~al.(2021)Papamakarios, Nalisnick, Rezende, Mohamed,
  and Lakshminarayanan]{papamakarios2021normalizing}
Papamakarios, G., Nalisnick, E.~T., Rezende, D.~J., Mohamed, S., and
  Lakshminarayanan, B.
\newblock Normalizing flows for probabilistic modeling and inference.
\newblock \emph{J. Mach. Learn. Res.}, 22\penalty0 (57):\penalty0 1--64, 2021.

\bibitem[{Papernot} \& {McDaniel}(2018){Papernot} and
  {McDaniel}]{2018arXiv180304765P}
{Papernot}, N. and {McDaniel}, P.
\newblock {Deep k-Nearest Neighbors: Towards Confident, Interpretable and
  Robust Deep Learning}.
\newblock \emph{arXiv e-prints}, art. arXiv:1803.04765, March 2018.
\newblock \doi{10.48550/arXiv.1803.04765}.

\bibitem[Perreault~Levasseur et~al.(2017)Perreault~Levasseur, Hezaveh, and
  Wechsler]{PerreaultLevasseur:2017ltk}
Perreault~Levasseur, L., Hezaveh, Y.~D., and Wechsler, R.~H.
\newblock {Uncertainties in Parameters Estimated with Neural Networks:
  Application to Strong Gravitational Lensing}.
\newblock \emph{Astrophys. J. Lett.}, 850\penalty0 (1):\penalty0 L7, 2017.
\newblock \doi{10.3847/2041-8213/aa9704}.

\bibitem[Prangle et~al.(2013)Prangle, Blum, Popovic, and
  Sisson]{prangle2013diagnostic}
Prangle, D., Blum, M., Popovic, G., and Sisson, S.
\newblock Diagnostic tools of approximate bayesian computation using the
  coverage property.” arxiv preprint.
\newblock \emph{arXiv preprint arXiv:1301.3166}, 412, 2013.

\bibitem[Price et~al.(2018)Price, Drovandi, Lee, and Nott]{price2018bayesian}
Price, L.~F., Drovandi, C.~C., Lee, A., and Nott, D.~J.
\newblock Bayesian synthetic likelihood.
\newblock \emph{Journal of Computational and Graphical Statistics}, 27\penalty0
  (1):\penalty0 1--11, 2018.

\bibitem[Pritchard et~al.(1999)Pritchard, Seielstad, Perez-Lezaun, and
  Feldman]{pritchard1999population}
Pritchard, J.~K., Seielstad, M.~T., Perez-Lezaun, A., and Feldman, M.~W.
\newblock Population growth of human y chromosomes: a study of y chromosome
  microsatellites.
\newblock \emph{Molecular biology and evolution}, 16\penalty0 (12):\penalty0
  1791--1798, 1999.

\bibitem[{Ramesh} et~al.(2022){Ramesh}, {Lueckmann}, {Boelts},
  {Tejero-Cantero}, {Greenberg}, {Gon{\c{c}}alves}, and
  {Macke}]{2022arXiv220306481R}
{Ramesh}, P., {Lueckmann}, J.-M., {Boelts}, J., {Tejero-Cantero}, {\'A}.,
  {Greenberg}, D.~S., {Gon{\c{c}}alves}, P.~J., and {Macke}, J.~H.
\newblock {GATSBI: Generative Adversarial Training for Simulation-Based
  Inference}.
\newblock \emph{arXiv e-prints}, art. arXiv:2203.06481, March 2022.
\newblock \doi{10.48550/arXiv.2203.06481}.

\bibitem[Rezende \& Mohamed(2015)Rezende and Mohamed]{rezende2015variational}
Rezende, D. and Mohamed, S.
\newblock Variational inference with normalizing flows.
\newblock In \emph{International conference on machine learning}, pp.\
  1530--1538. PMLR, 2015.

\bibitem[Rozet et~al.(2021)]{rozet2021arbitrary}
Rozet, F. et~al.
\newblock Arbitrary marginal neural ratio estimation for likelihood-free
  inference.
\newblock 2021.

\bibitem[Rubin(1984)]{10.1214/aos/1176346785}
Rubin, D.~B.
\newblock {Bayesianly Justifiable and Relevant Frequency Calculations for the
  Applied Statistician}.
\newblock \emph{The Annals of Statistics}, 12\penalty0 (4):\penalty0 1151 --
  1172, 1984.
\newblock \doi{10.1214/aos/1176346785}.
\newblock URL \url{https://doi.org/10.1214/aos/1176346785}.

\bibitem[Schall(2012)]{schall2012empirical}
Schall, R.
\newblock The empirical coverage of confidence intervals: Point estimates and
  confidence intervals for confidence levels.
\newblock \emph{Biometrical journal}, 54\penalty0 (4):\penalty0 537--551, 2012.

\bibitem[Skilling(2006)]{10.1214/06-BA127}
Skilling, J.
\newblock {Nested sampling for general Bayesian computation}.
\newblock \emph{Bayesian Analysis}, 1\penalty0 (4):\penalty0 833 -- 859, 2006.
\newblock \doi{10.1214/06-BA127}.
\newblock URL \url{https://doi.org/10.1214/06-BA127}.

\bibitem[Sohl{-}Dickstein et~al.(2015)Sohl{-}Dickstein, Weiss, Maheswaranathan,
  and Ganguli]{DBLP:journals/corr/Sohl-DicksteinW15}
Sohl{-}Dickstein, J., Weiss, E.~A., Maheswaranathan, N., and Ganguli, S.
\newblock Deep unsupervised learning using nonequilibrium thermodynamics.
\newblock \emph{CoRR}, abs/1503.03585, 2015.
\newblock URL \url{http://arxiv.org/abs/1503.03585}.

\bibitem[Song et~al.(2020)Song, Sohl-Dickstein, Kingma, Kumar, Ermon, and
  Poole]{song2020score}
Song, Y., Sohl-Dickstein, J., Kingma, D.~P., Kumar, A., Ermon, S., and Poole,
  B.
\newblock Score-based generative modeling through stochastic differential
  equations.
\newblock \emph{arXiv preprint arXiv:2011.13456}, 2020.

\bibitem[{Stone} \& {Courteau}(2019){Stone} and
  {Courteau}]{2019ApJ...882....6S}
{Stone}, C. and {Courteau}, S.
\newblock {The Intrinsic Scatter of the Radial Acceleration Relation}.
\newblock \emph{The Astrophysical Journal}, 882\penalty0 (1):\penalty0 6,
  September 2019.
\newblock \doi{10.3847/1538-4357/ab3126}.

\bibitem[{Stone} et~al.(2021){Stone}, {Courteau}, and
  {Arora}]{2021ApJ...912...41S}
{Stone}, C., {Courteau}, S., and {Arora}, N.
\newblock {The Intrinsic Scatter of Galaxy Scaling Relations}.
\newblock \emph{The Astrophysical Journal}, 912\penalty0 (1):\penalty0 41, May
  2021.
\newblock \doi{10.3847/1538-4357/abebe4}.

\bibitem[Talts et~al.(2018)Talts, Betancourt, Simpson, Vehtari, and
  Gelman]{talts2018validating}
Talts, S., Betancourt, M., Simpson, D., Vehtari, A., and Gelman, A.
\newblock Validating bayesian inference algorithms with simulation-based
  calibration.
\newblock \emph{arXiv preprint arXiv:1804.06788}, 2018.

\bibitem[{Tejero-Cantero} et~al.(2020){Tejero-Cantero}, {Boelts}, {Deistler},
  {Lueckmann}, {Durkan}, {Gon{\c{c}}alves}, {Greenberg}, and
  {Macke}]{2020JOSS....5.2505T}
{Tejero-Cantero}, A., {Boelts}, J., {Deistler}, M., {Lueckmann}, J.-M.,
  {Durkan}, C., {Gon{\c{c}}alves}, P., {Greenberg}, D., and {Macke}, J.
\newblock {sbi: A toolkit for simulation-based inference}.
\newblock \emph{The Journal of Open Source Software}, 5\penalty0 (52):\penalty0
  2505, August 2020.
\newblock \doi{10.21105/joss.02505}.

\bibitem[Thomas et~al.(2022)Thomas, Dutta, Corander, Kaski, and
  Gutmann]{thomas2022likelihood}
Thomas, O., Dutta, R., Corander, J., Kaski, S., and Gutmann, M.~U.
\newblock Likelihood-free inference by ratio estimation.
\newblock \emph{Bayesian Analysis}, 17\penalty0 (1):\penalty0 1--31, 2022.

\bibitem[Treu(2010)]{treu2010strong}
Treu, T.
\newblock Strong lensing by galaxies.
\newblock \emph{Annual Review of Astronomy and Astrophysics}, 48:\penalty0
  87--125, 2010.

\bibitem[Vincent(2011)]{vincent2011connection}
Vincent, P.
\newblock A connection between score matching and denoising autoencoders.
\newblock \emph{Neural computation}, 23\penalty0 (7):\penalty0 1661--1674,
  2011.

\bibitem[Wagner-Carena et~al.(2021)Wagner-Carena, Park, Birrer, Marshall,
  Roodman, and Wechsler]{Wagner-Carena:2020yun}
Wagner-Carena, S., Park, J.~W., Birrer, S., Marshall, P.~J., Roodman, A., and
  Wechsler, R.~H.
\newblock {Hierarchical Inference with Bayesian Neural Networks: An Application
  to Strong Gravitational Lensing}.
\newblock \emph{Astrophys. J.}, 909\penalty0 (2):\penalty0 187, 2021.
\newblock \doi{10.3847/1538-4357/abdf59}.

\bibitem[Wilson \& Izmailov(2020)Wilson and Izmailov]{NEURIPS2020_322f6246}
Wilson, A.~G. and Izmailov, P.
\newblock Bayesian deep learning and a probabilistic perspective of
  generalization.
\newblock In Larochelle, H., Ranzato, M., Hadsell, R., Balcan, M., and Lin, H.
  (eds.), \emph{Advances in Neural Information Processing Systems}, volume~33,
  pp.\  4697--4708. Curran Associates, Inc., 2020.
\newblock URL
  \url{https://proceedings.neurips.cc/paper/2020/file/322f62469c5e3c7dc3e58f5a4d1ea399-Paper.pdf}.

\bibitem[{Zhu} \& {Zabaras}(2018){Zhu} and {Zabaras}]{2018JCoPh.366..415Z}
{Zhu}, Y. and {Zabaras}, N.
\newblock {Bayesian deep convolutional encoder-decoder networks for surrogate
  modeling and uncertainty quantification}.
\newblock \emph{Journal of Computational Physics}, 366:\penalty0 415--447,
  August 2018.
\newblock \doi{10.1016/j.jcp.2018.04.018}.

\bibitem[Zuo et~al.(2020)Zuo, Chen, Li, Deng, Chen, Behler, Csányi, Shapeev,
  Thompson, Wood, and Ong]{doi:10.1021/acs.jpca.9b08723}
Zuo, Y., Chen, C., Li, X., Deng, Z., Chen, Y., Behler, J., Csányi, G.,
  Shapeev, A.~V., Thompson, A.~P., Wood, M.~A., and Ong, S.~P.
\newblock Performance and cost assessment of machine learning interatomic
  potentials.
\newblock \emph{The Journal of Physical Chemistry A}, 124\penalty0
  (4):\penalty0 731--745, 2020.
\newblock \doi{10.1021/acs.jpca.9b08723}.
\newblock URL \url{https://doi.org/10.1021/acs.jpca.9b08723}.
\newblock PMID: 31916773.

\end{thebibliography}

%%%%%%%%%%%%%%%%%%%%%%%%%%%%%%%%%%%%%%%%%%%%%%%%%%%%%%%%%%%%%%%%%%%%%%%%%%%%%%%
%%%%%%%%%%%%%%%%%%%%%%%%%%%%%%%%%%%%%%%%%%%%%%%%%%%%%%%%%%%%%%%%%%%%%%%%%%%%%%%
% APPENDIX
%%%%%%%%%%%%%%%%%%%%%%%%%%%%%%%%%%%%%%%%%%%%%%%%%%%%%%%%%%%%%%%%%%%%%%%%%%%%%%%
%%%%%%%%%%%%%%%%%%%%%%%%%%%%%%%%%%%%%%%%%%%%%%%%%%%%%%%%%%%%%%%%%%%%%%%%%%%%%%%
\newpage
\appendix
\onecolumn
\section{Connection between both definitions}\label{app:connection}

\begin{figure*}%[tbhp]
    \centering
    \includegraphics[width=.99\linewidth]{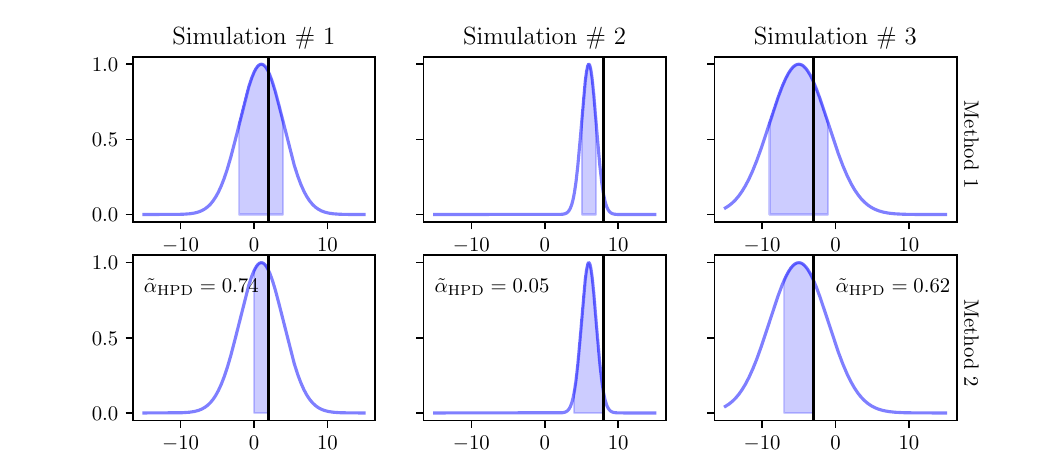}
    \caption{
    This figure illustrates the intuition behind the two ways of calculating high posterior density coverages. Each column shows one of three simulations in a toy example. The blue curves show the predicted posteriors, and the black vertical lines show the corresponding truths. We want to calculate the coverage for the $68 \%$ credibility level ($\alpha = 0.32$). The first approach, shown in the top row, would consist of calculating the $1 - \alpha$ credibility region, then checking how often the truth is in the said region (in this case twice, so the coverage is $2/3$). The alternative approach, described in~\cref{sec:prev_work}, and shown in the bottom row, is to find the HPD region defined by the truth, then find how often $\tilde{\alpha}_{\rm HPD} < \alpha$, which is again twice. The plot illustrates the fact that these two approaches are exactly equivalent. 
    }
    \label{fig:app_hpd}
\end{figure*}

\begin{figure*}%[tbhp]
    \centering
    \includegraphics[width=.99\linewidth]{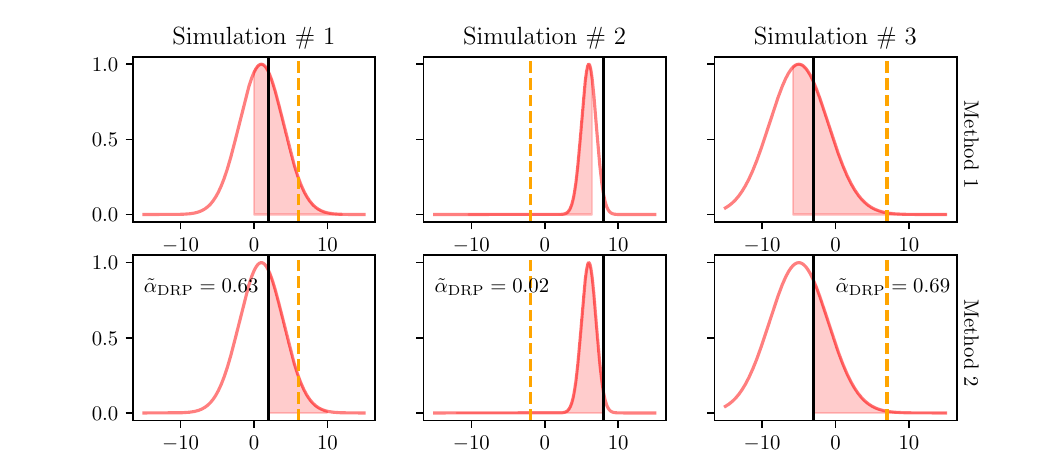}
    \caption{
    Similarly to~\cref{fig:app_hpd}, this figure illustrates the intuition behind the two ways of calculating `distance to random point' coverages. Each column shows one of three simulations in a toy example. The red curves show the predicted posteriors, and the black vertical lines show the corresponding truths. The orange lines show the randomly selected reference points. We want to calculate the coverage for the $68 \%$ credibility level ($\alpha = 0.32$). The first approach, shown in the top row, would consist of finding the $1 - \alpha$ credibility region centered around the reference point, then checking how often the truth is in the said region (in this case twice, so the coverage is $2/3$). The alternative approach, shown in the bottom row, is to find the TARP region defined by the reference and the truth, then find how often $\tilde{\alpha}_{\rm HPD} < \alpha$, which is again twice. The plot illustrates the fact that these two approaches are exactly equivalent. 
    }
    \label{fig:app_drp}
\end{figure*}

\cref{sec:prev_work} discussed the differences between two possible methods for calculating coverage probabilities, both for HPD and TARP regions. We try to build more intuition behind that connection in this appendix. Focusing first on the case of HPD regions, shown in~\cref{fig:app_hpd}. The first method, perhaps more intuitive but far more inefficient, would be to calculate the $1 - \alpha$ credibility region, then check how often the truth is in the said region for each simulation, and for multiple values of alpha (notice the nested loop). The second method, a consequence of the very important~\cref{lemma:hpd}, and already used by~\cref{alg:ecp_hpd} would be to find the HPD region defined by the truth for each simulation, and its corresponding credibility level $1 - \tilde{\alpha}_{\rm HPD}$. We can then calculate the coverage for the $1 - \alpha$ level as $\sum_{i=1}^N \tilde{\alpha}_{\rm HPD} \geq \alpha$, where $N$ is the number of simulations. 

A similar logic applies to TARP credibility regions. While we could find the radius from the reference point, such that $\alpha$ reaches a certain value, it is far more computationally efficient, and equivalent, to use the credibility regions defined by the true values, as shown in~\cref{fig:app_drp}. This is the method used by~\cref{alg:ecp_name}.

\section{Intuition about over and under confident plots}\label{app:intuition}

\begin{figure*}%[tbhp]
    \centering
    \includegraphics[width=.99\linewidth]{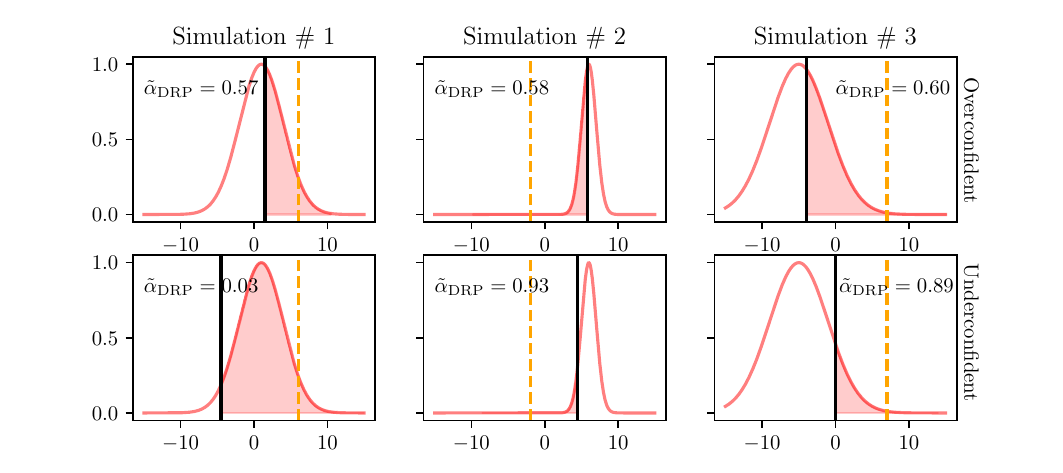}
    \caption{
    This figure illustrates the reason behind the shapes of over and underconfident curves obtained using TARP coverage, such as those in~\cref{fig:toy_gaussian}. The top shows three example simulations with overconfident predictions, while the bottom shows underconfident predictions. The figure shows that TARP coverages tend to be close to $0.5$, while for overconfident regions they tend to be close to either $0$ or $1$.
    }
    \label{fig:app_int}
\end{figure*}

Practitioners used to applying coverage probabilities to validate SBI analysis will be used to seeing over- and underconfident curves, such as those in the blue curves of ~\cref{fig:toy_gaussian}. However, the same figure shows how the TARP method produces different curves for over- and underconfident posterior estimators. The aim of this appendix is to provide some intuition behind these differences.

Firstly, we focus on underconfident posteriors, shown in the top panel of~\cref{fig:app_int}. In this case, we see that the TARP coverage tends to be close to $0.5$. This is because regardless of where the random reference point is, if the truth is close to the peak of the posterior, the TARP area is likely to cover approximately half of the distribution. On the other hand, for overconfident posteriors, shown in the bottom panel of~\cref{fig:app_int}, we see that the TARP coverage tends to be close to either $0$ or $1$. This is because regardless of where the random reference point is, if the truth is far from the peak of the posterior, the TARP area is likely to cover either the whole distribution, or none of it.

\section{Biased case experiment function}\label{app:biased_case}

\begin{figure*}%[tbhp]
    \centering
    \includegraphics[width=.39\linewidth]{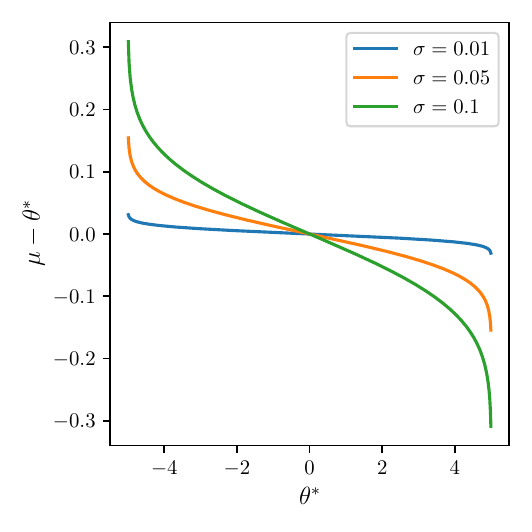}
    \caption{
    The position-dependent function used as the mean in the {\it biased case} in~\cref{sec:toy}, for three different values of sigma.
    }
    \label{fig:biased_case}
\end{figure*}

~\cref{fig:biased_case} shows the function~\cref{eq:biased}, used in~\cref{sec:toy} for the one-dimensional case, as the means of the normal distributions. The function shows how, when $\theta^*$ is zero, the distributions are centered at the correct value, whereas as we move away from zero, the posterior estimator will be increasingly biased. 

\section{Dependence on $\theta_r$ distribution and distance metric}\label{app:dependence}
\begin{figure*}[tbhp]
    \centering
    \includegraphics[width=.99\linewidth]{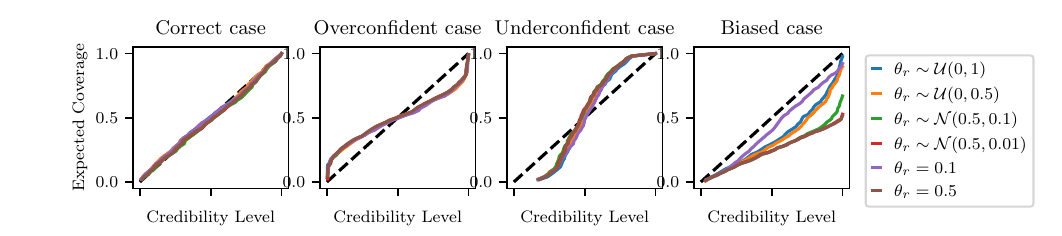}
    \caption{
    The Gaussian Toy Model experiment described in~\cref{sec:toy}, drawing the reference points $\theta_r$ from different distributions, as described in~\cref{sec:dependence}.
    }
    \label{fig:references}
\end{figure*}

\begin{figure*}[tbhp]
    \centering
    \includegraphics[width=.99\linewidth]{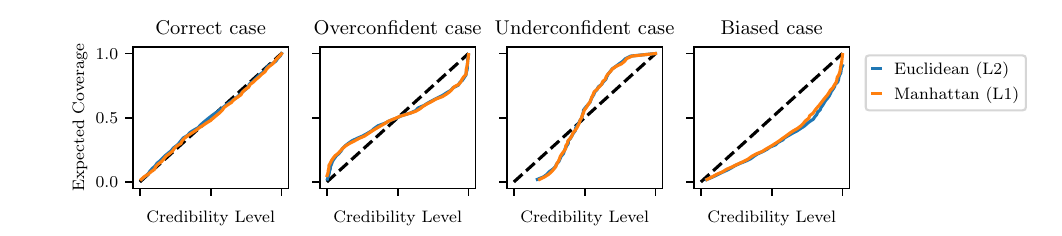}
    \caption{
    The Gaussian Toy Model experiment described in~\cref{sec:toy}, using L1 or L2 distance metrics, as described in~\cref{sec:dependence}.
    }
    \label{fig:distances}
\end{figure*}

\cref{fig:references} shows the same as~\cref{fig:toy_gaussian}, but varying the distribution used to draw $\theta_r$. We find that this only makes a difference in the biased case, but even then there is clear evidence of bias for all distributions. \cref{fig:distances} shows the same, comparing the use of L1 and L2 as distance metrics. We find no appreciable differences in this case. We, therefore, conclude that our method is robust to choices of $\theta_r$ distribution and distance metric.

In~\cref{sec:posterior-prior}, we discussed how when we have a distribution $\tilde{p}(\theta_r | x)$ that has some dependency on $x$, the TARP method reveals an inaccurate posterior estimator, in the case when the posterior estimator is simply recovering the prior.~\cref{fig:prior_references} shows what happens to this experiment for different distributions. We see how the distributions that do not depend on $x$, shown in continuous lines, do not detect the inaccurate posterior estimator as expected. On the other hand, the distributions that depend strongly on $x$, shown as dotted lines, very clearly detect the inaccurate estimator. Finally, we show a distribution with a weaker dependence on $x$, where TARP does lie away from the diagonal line, but much closer than the other $x$-dependent distributions, as expected.

\section{Gravitational lensing experiment details}\label{app:lensing}

\begin{figure*}[tbhp]
    \centering
    \includegraphics[width=.49\linewidth]{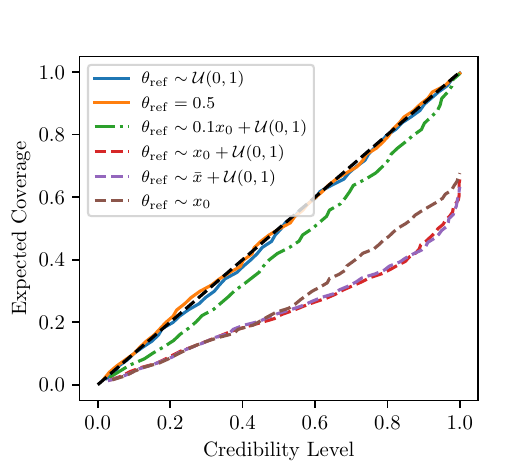}
    \caption{
    The same expected coverage vs credibility level plot shown in~\cref{fig:uninformative}, for different $\tilde{p}(\theta_r | x)$ distributions. The continuous lines show distributions that do not depend on $x$, the dash-dotted line shows a distribution with a weak dependence on $x$, and the dotted lines show a stronger dependence.
    }
    \label{fig:prior_references}
\end{figure*}

As shown in \citet{adam2022posterior}, gravitational lensing source reconstruction can be performed using techniques from score-based modeling. Here we summarize the key ideas behind score-based modeling and how we generate biased and exact posterior samples.

Score-based modeling works by perturbing a training dataset sampled from a prior $p(\theta)$ with noise of increasing scales indexed by $t \in [0, T]$. Here $t = 0$ corresponds to unperturbed data ($p_0(\theta) = p(\theta)$) and $t=T$ corresponds to perturbing the data so much it is buried under noise and follows a Gaussian distribution ($p_T(\theta) = \mathcal{N}(\theta | 0, \sigma_T^2)$). The noising process be described by the stochastic differential equation (SDE) \citep{song2020score}
\begin{equation}
    \dd{\theta}_t = f(t, \theta) \dd{t} + g(t) \dd{w} \, ,
\end{equation}
where $w$ is a standard Wiener process. Using denoising score-matching~\citep{hyvarinen2005estimation, vincent2011connection, song2020score}, a neural network can be trained to approximate the time-dependent prior score $\grad_\theta p_t(\theta)$, where $p_t(\theta)$ is the distribution over data perturbed by the noising process up to time $t$. Given the prior score, samples can be generated by solving the corresponding reverse SDE (RSDE) backward in time, starting with samples from $p_T$:
\begin{equation}
    \dd{\theta} = \left[ f(t, \theta) - g^2(t) \grad_\theta \log p_t(\theta) \right] \dd{t} + g(t) \dd{w} \, ,
\end{equation}
where here $\dd{t}$ is a negative timestep.

For simplicity, instead of fitting a score-based model, we fit a multivariate Gaussian to the PROBES dataset of galaxy images as our prior, giving $p(\theta) = \mathcal{N}(\mu_0, \Sigma_0)$. We use the variance-exploding SDE from \citet{song2020score} as our noise process. The prior at time $t$ is thus $p_t(\theta) = \mathcal{N}(\theta | \mu_0, \Sigma_0 + \sigma_t^2 \mathbb{I})$, where $\sigma_t^2$ is the variance of the noise process at time $t$. This expression can be used to evaluate the prior score analytically.

To modify the sampling procedure to generate samples from $p(\theta | x)$ for some observation $x$, we must condition the score in the RSDE, replacing the prior score with the posterior score:
\begin{equation}
    \dd{\theta} = \left[ f(t, \theta) - g^2(t) \grad_\theta \log p_t(\theta | x) \right] \dd{t} + g(t) \dd{w} \, , \label{eq:cond-rsde}
\end{equation}
By Bayes' rule, the posterior score is
\begin{equation}
    \grad_\theta \log p_t(\theta | x) = \grad_\theta \log p_t(x | \theta) + \grad_\theta \log p_t(\theta) \, ,
\end{equation}
where the first term on the RHS is the score of the likelihood. As pointed out in \citet{adam2022posterior}, this time-dependent likelihood is in general intractable but can be approximated as
\begin{equation}
    \hat{p}_t(x | \theta) = \mathcal{N}(x | A \theta, \sigma_n^2 \mathbb{I} + \sigma_t^2 A A^T) \approx p_t(x | \theta) \, , \label{eq:gl-biased}
\end{equation}
where the matrix $A$ encodes the lensing distortions and $\sigma_n$ is the standard deviation of the noise in the observation (see \cref{sec:gl}). \emph{However}, when $p(\theta)$ is a multivariate Gaussian, the time-dependent likelihood \emph{is tractable}, evaluating to
\begin{equation}
    p_t(x | \theta) = \mathcal{N}(x | A \theta_c(\theta), \sigma_n^2 \mathbb{I} + A \Sigma_c A^T) \, , \label{eq:gl-exact}
\end{equation}
where
\begin{equation}
    \Sigma_c := \left( \Sigma_0^{-1} + \sigma_t^{-2} \mathbb{I} \right)^{-1} \, , \qq{} \theta_c(\theta) := \sigma_t^{-2} \Sigma_c \theta \, .
\end{equation}
We, therefore, have two methods for sampling the posterior for the source galaxy's light: solving the RSDE \cref{eq:cond-rsde} using the exact time-dependent likelihood (\cref{eq:gl-exact}) or the approximate, biased one (\cref{eq:gl-biased}). We refer to these as the `exact' and `biased' samplers respectively.

Finally, we solve both the exact and biased RSDEs by discretizing with the Euler-Maruyama method (see e.g.\, \citet{song2020score}). We find 300 steps are sufficient to ensure convergence.

\end{document}